\documentclass{article}
\usepackage[main, final]{neurips_2026}

\usepackage{amsmath,amsfonts,bm}

\def\eqref#1{equation~\ref{#1}}

\def\1{\bm{1}}

\DeclareMathAlphabet{\mathsfit}{\encodingdefault}{\sfdefault}{m}{sl}
\SetMathAlphabet{\mathsfit}{bold}{\encodingdefault}{\sfdefault}{bx}{n}

\DeclareMathOperator*{\argmax}{arg\,max}

\usepackage[T1]{fontenc}
\usepackage{hyperref}
\usepackage{url}
\hypersetup{hidelinks, pdftitle={Calibration Is Not Control: Intervention Value for LLM-Agent Oversight}, pdfauthor={Chubin Zhang, Zhenglin Wan, Xingrui Yu, Jingxuan Wu, Qi Wen, Pengfei Zhou, Wangbo Zhao, Ivor Tsang}}
\usepackage{amsmath,amssymb,amsthm}
\usepackage{booktabs}
\usepackage{graphicx}
\usepackage{xcolor}
\usepackage{colortbl}
\usepackage{wrapfig}
\usepackage{array}
\usepackage{microtype}
\usepackage{tabularx}
\usepackage{needspace}
\usepackage{multirow,subcaption}
\theoremstyle{definition}
\newtheorem{definition}{Definition}
\theoremstyle{plain}
\newtheorem{proposition}{Proposition}
\theoremstyle{remark}
\newtheorem{remark}{Remark}
\newcommand{\cont}{\mathrm{continue}}
\newcommand{\pfail}{p_{\mathrm{fail}}}
\newcommand{\EE}{\mathbb{E}}\newcommand{\cA}{\mathcal{A}}\newcommand{\tauadv}{\tau}\newcommand{\Gap}{\mathrm{Gap}}\newcommand{\NMG}{\mathrm{NMG}}
\usepackage[capitalize]{cleveref}
\newtheorem{corollary}{Corollary}
\renewcommand{\arraystretch}{1.12}
\definecolor{tablegray}{gray}{0.95}

\title{Calibration Is Not Control:\\ Intervention Value for LLM-Agent Oversight}

\author{%
  \textbf{Chubin Zhang}$^{1}$ \quad \textbf{Zhenglin Wan}$^{2}$ \quad \textbf{Xingrui Yu}$^{3,4,\ddagger}$ \quad \textbf{Jingxuan Wu}$^{5}$ \\
  \textbf{Qi Wen}$^{2}$ \quad \textbf{Pengfei Zhou}$^{2}$ \quad \textbf{Wangbo Zhao}$^{2}$ \quad \textbf{Ivor Tsang}$^{1,3,4}$ \\[0.7em]
  {\small $^{1}$Nanyang Technological University, Singapore \quad $^{2}$National University of Singapore, Singapore} \\
  {\small $^{3}$CFAR, Agency for Science, Technology and Research, Singapore} \\
  {\small $^{4}$IHPC, Agency for Science, Technology and Research, Singapore} \\
  {\small $^{5}$Department of Statistics and Operations Research, UNC-Chapel Hill, United States} \\[0.4em]
  {\small $^{\ddagger}$Corresponding author: \texttt{yu\_xingrui@a-star.edu.sg}}
}

\begin{document}
\maketitle
\suppressfloats[t]

\begin{abstract}
Runtime oversight often intervenes when an LLM agent's calibrated failure score crosses a threshold. Yet states with the same failure risk can differ in whether intervention helps. Strictly increasing recalibration preserves the threshold policy class and cannot recover this distinction. We formalize when a summary is sufficient for intervention decisions and the utility lost when it is not. We evaluate the consequences by replaying agent prefixes and executing alternative actions from the same state. On ALFWorld, holding features, estimator, and router fixed while changing the supervision target from failure to intervention utility lowers regret from 0.51 to 0.09; the gain replicates on a second suite of mid-episode prefixes. A deployable intervention-trained scalar also beats the failure-score threshold rule selected on test outcomes. Online, on 300 unseen tasks with a fixed stronger-model handoff, a frozen prefix-feature controller improves utility over failure-triggered routing, handing off less often (35\% vs 48\%) and succeeding more often (45\% vs 37\%). Gains depend on intervention value and are small on two reasoning benchmarks. Oversight signals should be evaluated by the decisions they support alongside their predictive quality. Code is available at \url{https://github.com/bennidict23/calibration-is-not-control}.

\end{abstract}

\section{Introduction}
\label{sec:intro}
\begin{figure}[t]
\centering
\includegraphics[width=\linewidth]{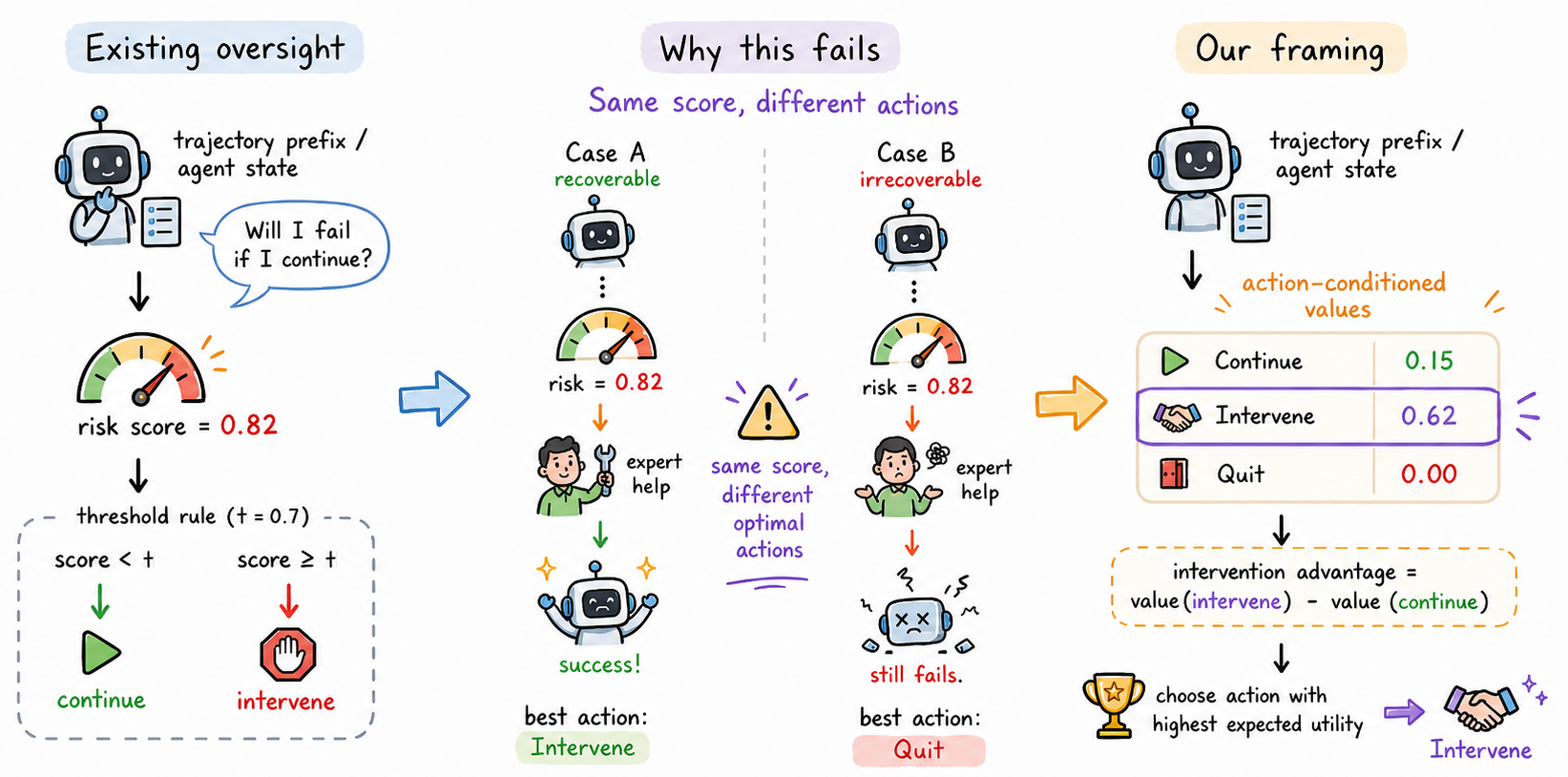}
\caption{\textbf{The same risk can call for different actions.} Left: failure-triggered oversight scores continuation risk and intervenes when the score crosses a threshold $t$. Middle: Cases A and B have the same risk score, but a handoff can rescue only A, so the best action is to intervene in A and to quit in B. Right: for a generic state, the decision compares the expected utility of each available action; the intervention advantage is the difference between intervening and continuing. Values are illustrative, with quitting worth zero as in our utilities.}
\label{fig:overview}
\end{figure}

An oversight system watches an LLM agent's trajectory and decides whether to continue, stop, or intervene. A common approach scores continuation risk, as a failure probability, a confidence, or an uncertainty estimate, and intervenes when that score crosses a threshold. Recent work improves these scores' calibration \citep{kadavath2022language,xuan2026confidence_dichotomy,zhang2026agentic_calibration}, and routing interventions on them is common practice \citep{ren2023knowno,ding2026calibrate_then_act}. Consider two states assigned the same failure score, with a handoff available (Figure~\ref{fig:overview}, middle). In one, the handoff can rescue the episode; in the other, the episode can no longer succeed, whoever acts next. Intervention is worth its cost in the first and worse than quitting in the second. A system that routes on the shared score must either waste the intervention on the second or let the first run to failure.

\Needspace{5\baselineskip}
The score is not wrong. It is answering a different question. Failure probability describes what happens if the agent continues. The intervention decision depends on what changes if it does not: at prefix $h$, the advantage of intervention $i$ is $\tau_i(h)=Q_i(h)-Q_{\cont}(h)$, the difference in expected downstream utility. This settles the binary choice; with more actions the controller compares all expected utilities. A summary supports the decision only if it preserves that comparison. Recalibration cannot restore a distinction the score lacks: a strictly increasing transform leaves the score's order, and the threshold policies it can express, unchanged (Remark~\ref{rem:recal}).

Two lines of work come close. Learning to defer puts the expert's accuracy into the deferral objective \citep{mozannar2020consistent}, and \citet{vasudev2026failure_prevention} show for LLM agents that accurate failure prediction can coexist with harmful intervention. Neither measures how much of a risk score's control loss is recovered at the same state by changing what the monitor is trained on, or whether a controller trained on intervention outcomes keeps its lead online. A clean answer requires three things: outcomes for the actions not taken, at the same state; a pipeline in which only the supervision target changes; and a deployment with a real intervention, where the controller's own decisions shape the trajectory.

We meet the first requirement by prefix branching. In a replayable environment we reconstruct a decision prefix, execute continuation and every available alternative from it, and record their downstream utilities (Section~\ref{sec:branching}). These outcomes supervise controllers that see only the prefix; the deployed controller never runs a trial intervention. The matched comparison of Section~\ref{sec:ladder} meets the second requirement, and the closed-loop suite of Section~\ref{sec:closedloop} the third. Section~\ref{sec:results} establishes five findings.

\textbf{Failure-triggered routing leaves most utility unclaimed.} Across 26 ALFWorld runs, its regret is 0.506 against 0.110 for a prefix-feature controller, which reads the prefix features directly, and failure-score thresholds and action order chosen on the test outcomes still trail a single intervention-trained scalar (0.275 vs.\ 0.229; Section~\ref{sec:gap}).

\textbf{The supervision target accounts for the loss.} Holding features, estimator, and router fixed, changing only the label from failure to intervention utility moves regret from 0.513 to 0.094, and the gain replicates on a second suite of mid-episode prefixes (Section~\ref{sec:ladder}).

\textbf{Calibration cannot recover it.} Above a failure score of 0.7, the best action is to intervene on 46\% of prefixes and to quit on 48\%, interleaved along the score; Platt scaling cuts the confidence score's calibration error from 0.463 to 0.006 and leaves its control regret unchanged (Section~\ref{sec:mechanism}).

\textbf{The gain survives a real handoff.} On 300 unseen tasks with GPT-5.4 as the handoff, the prefix-feature controller hands off on 35\% rather than 48\% of episodes and succeeds on 45\% rather than 37\% (Section~\ref{sec:closedloop}).

\textbf{The gap depends on the regime, and its supervision is affordable.} It is largest on ALFWorld, where the expert can rescue many high-risk prefixes, smaller on ScienceWorld, and small on two reasoning benchmarks, where a gold-answer probe separates a weak intervention from a score that already decides (Section~\ref{sec:boundaries}); a quarter of the intervention branches retains 92\% of the gain (Section~\ref{sec:cost}).

Section~\ref{sec:theory} formalizes when a summary is sufficient for control and what is lost when it is not. This criterion, the branching protocol, and the measurements above are our contributions.

\Needspace{165pt}
\section{When Is a Summary Sufficient for Control?}
\label{sec:theory}

This section makes the introduction's claim precise: what a summary must preserve for the decision to be lossless, and how much is lost when it does not.

\subsection{Prediction targets and control targets}

At each prefix state $h\in\mathcal H$ of an agent trajectory, a controller chooses from $\mathcal A=\{\cont,a_1,\dots,a_K\}$, where the $a_k$ are non-continue actions such as handoff, verification, repair, or quitting. Let $U_a(h)$ be the random downstream utility after choosing $a$ at $h$, costs included, and $Q_a(h)=\mathbb{E}[U_a(h)\mid h]$. Two quantities are central:
\begin{equation}
\pfail(h)=\Pr\!\big(U_{\cont}(h)\le 0\mid h\big),
\qquad
\tau_i(h)=Q_i(h)-Q_{\cont}(h).
\label{eq:targets}
\end{equation}
The first asks whether continuation is likely to fail. The second asks whether intervention $i$ improves expected utility over continuing. We call $Q_i(h)$ the intervention value at $h$ and $\tau_i(h)$ its advantage over continuing; $\pfail(h)$ describes the continuation branch alone. The optimal action is $A^*(h)=\argmax_{a\in\mathcal A}Q_a(h)$; for the binary set $\{\cont,\mathrm{intervene}\}$ it is determined by the sign of $\tau(h)$, and with quitting or several interventions available it selects among all $Q_a(h)$ (the main experiments use three actions).

A summary $g(h)$ carries enough for the decision when conditioning on it never forces the controller to merge states whose optimal actions differ; we take this as the definition.

\begin{definition}[Sufficiency]
\label{def:suff}
A measurable summary $g:\mathcal H\to\mathbb{R}^d$ is sufficient for control if there exists a measurable $d$ with $\Pr_{h\sim\mathcal D}[d(g(h))\in A^*(h)]=1$.
\end{definition}

Failure-triggered routing is the special case $d(\pfail(h))$, and routing on a confidence or an uncertainty score is the same special case with a different $g$. In Figure~\ref{fig:overview}, Cases A and B share a risk score of 0.82 while their optimal actions differ, so no rule applied to that score is right at both.

\begin{proposition}[Lossless routing, binary case]
\label{prop:lossless}
Under $\mathcal A=\{\cont,\mathrm{intervene}\}$, $g$ is sufficient if and only if the sign of $\tau(h)$ is recoverable from $g(h)$ up to the tie set $\{\tau=0\}$: there is a $g$-measurable region $E$ with $\{\tau>0\}\subseteq E\subseteq\{\tau\ge 0\}$ up to $\mathcal D$-null sets.
\end{proposition}

A summary need not represent the trajectory; it must preserve the decision boundary induced by intervention advantage (proof in Appendix~\ref{app:proofs}). A failure score can correlate with $\tau$ and still fail this test, because $\tau$ depends on how the intervention interacts with the state, which $\pfail$ does not describe.

\subsection{The value lost when sufficiency fails}
\label{sec:regimes}

Proposition~\ref{prop:lossless} is all-or-nothing: it says whether a summary can be lossless, not how much is lost when it cannot. Let $V^*=\mathbb{E}[\max_a Q_a(h)]$ be the value of the fully informed controller and $V_g=\mathbb{E}[\max_a\mathbb{E}[Q_a(h)\mid g(h)]]$ the value of the best controller measurable with respect to $g$. The \emph{abstraction loss} is
\begin{equation}
\mathrm{Gap}(g)=V^*-V_g\ge 0.
\label{eq:gap}
\end{equation}
It is zero if and only if $g$ is sufficient. It is not an estimation error: a controller with perfect conditional expectations incurs it because it must average over states that $g$ cannot tell apart, such as Cases A and B of Figure~\ref{fig:overview} under the failure score. Appendix~\ref{app:proofs} lower-bounds it by the mass with different optimal actions that shares a value of $g$ (Proposition~\ref{prop:conflict}).

Recalibration is the natural response to a lossy score; Remark~\ref{rem:recal} says what it can and cannot change.

\begin{remark}[Recalibration preserves the threshold policy class]
\label{rem:recal}
For a fixed score $g$ and a strictly increasing continuous recalibration $T$, $\{h: T(g(h))>T(t)\}=\{h:g(h)>t\}$. With thresholds chosen freely, the threshold policies expressible on $T\circ g$ are the same as on $g$, and so is the best value within that class; the same holds for several thresholds with a fixed action ordering. It concerns representable policies, not a fixed numerical threshold, a finite-sample threshold search, or rules that use probabilities and costs directly (Appendix~\ref{app:proofs}).
\end{remark}

How much $\mathrm{Gap}(g)$ costs in practice depends on two conditions. Some alternative must have positive advantage on a non-null set, since otherwise no controller gains from deviating from continuation, and $g$ must merge states with different optimal actions, since otherwise some $g$-measurable routing is lossless (threshold routing also needs the action regions to be ordered in $g$). Section~\ref{sec:results} examines both on branched data.

\section{Prefix Branching}
\label{sec:branching}

\begin{figure}[t]
\centering
\includegraphics[width=\linewidth]{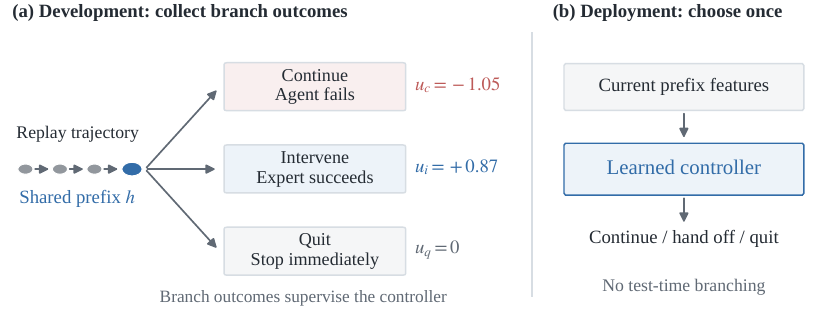}
\caption{\textbf{Collect alternatives offline; choose one action online.} (a) Replay to a shared prefix and execute each branch to obtain supervision; the utilities are illustrative and correspond to Case A of Figure~\ref{fig:overview}. (b) The deployed controller sees prefix features and selects one action, without executing trial branches.}
\label{fig:pipeline}
\end{figure}
Measuring what a summary loses requires knowing, at the same state, what each action would have yielded. Ordinary logs reveal only the outcome of the action taken, so they cannot identify $\tau(h)$ without assumptions that would themselves become suspects in the comparison. Prefix branching removes that dependence by supplying the missing outcomes directly, in environments that can be replayed (Figure~\ref{fig:pipeline}). We collect base trajectories, select decision prefixes, and replay the agent to each prefix, keeping a replay only if it agrees with the original run step by step (retained interactive data match at 100\%). From each verified prefix we execute every candidate action, so the same prefix yields outcomes for continue, the deployed intervention, and quit, which define the oracle action and the regret used throughout. Figure~\ref{fig:pipeline} traces Case A of Figure~\ref{fig:overview} through this protocol: continuation fails, the handoff succeeds, and quitting yields zero, so the oracle action is to intervene; in Case B the handoff also fails, and the oracle action is to quit.

Branching is a development-time protocol; the deployed controller does not branch. Prefixes are split into training, validation, and test sets, with all prefixes from one base trajectory assigned to the same split. In the ALFWorld suite no test task appears in any run's training or validation split, within or across models and seeds; the 26 runs cover 84 distinct held-out tasks, each used by 3.4 runs on average (Appendix~\ref{app:splits}).

Each branch is a single execution, so a realized branch utility is one sample of $U_a(h)$; repeating both actions five times on 72 ALFWorld and 72 ScienceWorld test prefixes changes the oracle action on 4.2\% and 8.3\% of prefixes, mostly near ties, and the main regret reductions by 0.001 (Appendix~\ref{app:repeated_branching}).

\section{Controllers and Offline Evaluation}
\label{sec:setup}

\subsection{From branch outcomes to deployed decisions}

Branching gives each development prefix $h$ a realized utility $u_a(h)$ per evaluated action, with the intervention and the execution policy after each action fixed within a suite. Predictors are fit on training prefixes, configurations selected on validation utility, and test branch outcomes reserved for evaluation.

Because the comparison is between summaries, we keep the controller family simple: random-forest predictors combined by one of four decision rules (direct value, explicit expected utility, expected utility with a lower confidence bound, and a two-stage rule), with the rule selected per run on validation (Appendix~\ref{app:features}); a stronger architecture would blur whether a gain came from the summary or from the model that reads it. Within that family we distinguish a controller's prediction target, its deployed input, and its action-selection rule. A continuation-risk predictor estimates $\pfail(h)=\Pr(U_{\cont}\leq 0\mid h)$; we write $\hat p_{\mathrm{fail}}(h)$ for its learned output. An intervention-aligned predictor estimates an outcome or utility quantity involving an available intervention. Intervention-success probability, intervention utility, continuation utility, and their difference are distinct targets; $\pfail$ is a function of the continuation branch alone, so it carries one term of $\tauadv$ and not the other. A compact summary combines a few predicted quantities; a prefix-feature controller uses the prefix feature vector directly. All summaries and the prefix-feature controller use prefix-available features and receive no intervention preview; targets and feature lists are given in Appendix~\ref{app:features}.

\paragraph{Benchmarks, agents, and interventions.}
ALFWorld \citep{shridhar2021alfworld} and ScienceWorld \citep{wang2022scienceworld} are interactive environments where an intervention changes the environment state; GSM8K \citep{cobbe2021training} and HotpotQA \citep{yang2018hotpotqa} are reasoning tasks where it changes only the reasoning branch. Deployed interventions differ in strength by design, so that the two conditions of Section~\ref{sec:regimes} can fail separately rather than together: defer-to-expert (ALFWorld), gold-prefix handoff (ScienceWorld), and same-model verify-and-reanswer (GSM8K, HotpotQA), defined in Appendix~\ref{app:interventions}. Base agents are Qwen2.5-7B-Instruct \citep{qwen2025qwen25} and Llama-3.1-8B-Instruct \citep{dubey2024llama3}. Utility is outcome score minus intervention and per-step costs, with quit worth zero (Table~\ref{tab:rewards}); ScienceWorld's continuous score gives the failure label $U_{\cont}\le 0$, with a boundary sweep in Appendix~\ref{app:swboundary}.

\subsection{Comparison protocol and offline regret}
\label{sec:protocols}\label{sec:regret}

Every deployable pipeline is trained on the same prefixes, tuned on the same validation split, and evaluated on the same held-out branches, so differences between pipelines are paired. Validation splits are small (4 to 20 prefixes per run), so selection is noisy; the matched comparison of Section~\ref{sec:ladder} therefore fixes the pipeline and gives every target the same two-threshold search (Appendix~\ref{app:matched}). For a frozen policy $\pi$, the offline evaluation measures regret relative to the best realized branch:
\begin{equation}
\widehat R(\pi)=\frac{1}{N}\sum_{j=1}^{N}
\left[\max_{a\in\mathcal A}u_a(h_j)-u_{\pi(h_j)}(h_j)\right].
\label{eq:regret}
\end{equation}
Regret is computed within each run and aggregated per suite (Appendix~\ref{app:stats}). Since $\mathbb{E}[\max_a U_a\mid h]\geq\max_a\mathbb{E}[U_a\mid h]$, the realized oracle inflates absolute regret relative to Bayes regret for policies independent of the test outcomes; the common term cancels in differences between policies evaluated on the same branches, so we report paired differences with their confidence intervals; Appendix~\ref{app:protocolmap} lists the protocol behind each table.

\section{Results}
\label{sec:results}

The offline comparisons of Sections~\ref{sec:gap} to~\ref{sec:mechanism} use an oracle expert as the intervention, so that intervention quality does not limit what a summary can recover; Section~\ref{sec:closedloop} replaces it with a real handoff.

\subsection{Failure-triggered routing leaves most utility unclaimed}
\label{sec:gap}

\begin{table}[!htb]
\caption{\textbf{Comparison ladder on the main ALFWorld suite.} The 26 runs share evaluation prefixes and an oracle-expert intervention. Deployable rows select on validation; the shaded reference selects thresholds and action order on test outcomes. Inputs and controllers vary here; Table~\ref{tab:matched} holds the scalar pipeline fixed.}
\label{tab:ladder}
\centering\small\renewcommand{\arraystretch}{1.18}
\begin{tabularx}{\linewidth}{l>{\raggedright\arraybackslash}Xr}
\toprule
Deployed input & Controller / selection & Regret $\downarrow$ \\
\midrule
Failure score & Two-threshold rule & 0.506 \\
Failure score & Random forest with lower-confidence adjustment & 0.449 \\
\rowcolor{tablegray}\textit{Failure score} & \textit{Two thresholds + action order, selected on test} & \textit{0.275} \\
\midrule
Intervention-aligned scalar & Prefix-only predictor and routing rule & 0.229 \\
Compact summary & Predictor combination and routing rule & 0.161 \\
Prefix feature vector & Action-conditioned controller family & \textbf{0.110} \\
\bottomrule
\end{tabularx}
\end{table}

On the main ALFWorld suite, failure-triggered routing has regret 0.506 and the prefix-feature controller 0.110 (Table~\ref{tab:ladder}). The reduction is 0.396, with a 95\% interval of $[0.310, 0.488]$ over runs; weighting every held-out prefix equally, it is 0.358, with a task-clustered interval of $[0.315, 0.404]$ over the 84 held-out tasks. The failure score predicts outcomes well (AUROC 0.74, calibration error 0.06), and better thresholds do not close the gap: a random-forest controller on the same score lowers regret only to 0.449, and thresholds and action order chosen on the test outcomes reach 0.275, still above a single intervention-trained scalar selected on validation (0.229). A compact summary of a few predicted quantities reaches 0.161, a reduction of 0.345 $[0.279, 0.420]$, and prefix-feature control improves on it by a further 0.051 $[0.028, 0.078]$.

\subsection{The supervision target accounts for the loss}
\label{sec:ladder}

The pipelines in Table~\ref{tab:ladder} differ in input, controller, and selection at once, so the gap could come from any of them. Table~\ref{tab:matched} holds the scalar pipeline fixed, with the same 26 features, random-forest regressor, and validation-selected two-threshold router, and changes only the supervision label.

\begin{table}[!htb]
\caption{\textbf{Changing the supervision target under a fixed pipeline.} All scalar rows use the same 26 features, random-forest regressor, and validation-selected two-threshold router; the last row uses a pair router. Means are over runs; $\Delta$ is the reduction from the failure label, with paired-bootstrap 95\% intervals. Protocols: Appendix~\ref{app:matched} and~\ref{app:midprefix}.}
\label{tab:matched}
\centering\small
\setlength{\tabcolsep}{4pt}\renewcommand{\arraystretch}{1.17}
\begin{tabular*}{\linewidth}{@{\extracolsep{\fill}}lrrrr@{}}
\toprule
& \multicolumn{2}{c}{Main suite, 26 runs} & \multicolumn{2}{c}{Mid-episode, 8 runs} \\
\cmidrule(lr){2-3}\cmidrule(l){4-5}
Supervision target & Regret $\downarrow$ & $\Delta$ [95\% CI] & Regret $\downarrow$ & $\Delta$ [95\% CI] \\
\midrule
Failure label & 0.513 & --- & 0.295 & --- \\
Continuation utility & 0.451 & 0.062 $[-0.051, 0.175]$ & 0.245 & 0.049 $[-0.005, 0.110]$ \\
\addlinespace[4pt]
Intervention success & 0.124 & 0.389 $[0.297, 0.472]$ & \textbf{0.081} & 0.214 $[0.178, 0.257]$ \\
Advantage $U_i-U_{\cont}$ & 0.152 & 0.361 $[0.289, 0.434]$ & 0.148 & 0.147 $[0.077, 0.222]$ \\
Intervention utility $U_i$ & 0.094 & 0.419 $[0.340, 0.501]$ & 0.098 & 0.196 $[0.165, 0.241]$ \\
\midrule
Both values (pair router) & \textbf{0.069} & 0.445 $[0.397, 0.493]$ & 0.095 & 0.199 $[0.152, 0.255]$ \\
\bottomrule
\end{tabular*}
\end{table}

On the main suite, continuation-side targets leave regret at 0.513 (failure) and 0.451 (utility). Every intervention-side scalar removes at least 0.36: intervention success reaches 0.124, advantage 0.152, and intervention utility 0.094. Improvements occur in 24 to 25 of 26 runs, with intervals over runs and over the 84 held-out tasks excluding zero (classifier variant in Appendix~\ref{app:matched}). With the pipeline fixed, the target accounts for the gap.

A single difference cannot express quitting, which is optimal only when both values are non-positive. Here the expert's value tracks recoverability, so low intervention value also signals a failing continuation. This is consistent with intervention utility beating advantage and both values doing best. Every intervention-side target keeps its edge across noisy experts; with the expert truncated to two or three steps, only the two-value pair still improves on the failure label (Appendix~\ref{app:matched}).

The right half of Table~\ref{tab:matched} repeats the ablation at six evenly spaced prefixes per trajectory, mostly mid-episode, with 57 to 60 validation prefixes per run (Appendix~\ref{app:midprefix}). The gain from intervention-side supervision replicates: failure-triggered routing has regret 0.286, intervening at every prefix 0.263, and the validation-selected controller 0.088, improving all eight runs.

\subsection{Calibration cannot separate what the score merges}
\label{sec:mechanism}

Calibrating the failure score does not change which prefixes it groups together. Figure~\ref{fig:mechanism} plots the branched outcomes behind Table~\ref{tab:ladder}. Above a learned failure score of 0.7, prefixes where intervention pays and prefixes where it fails are interleaved along the score axis: the best action is to intervene on 46\% of them and to quit on 48\%. Cases A and B of Figure~\ref{fig:overview} sketch one such pair. A strictly increasing recalibration only relabels the horizontal axis. It moves the points without reordering them, so the interleaving, and the threshold policies the score can express, stay the same (Remark~\ref{rem:recal}).

\begin{figure}[!htb]
\centering
\includegraphics[width=\linewidth]{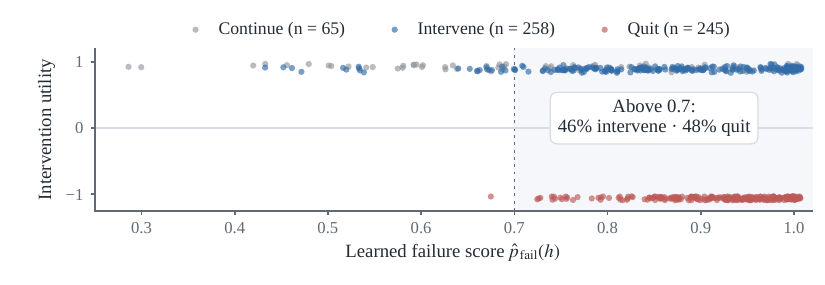}
\caption{\textbf{High failure risk hides opposite intervention outcomes.} ALFWorld main-suite prefixes, colored by the realized-best action and jittered for visibility. Above a learned failure score of 0.7, the best action is intervene on 46\% of prefixes and quit on 48\%. A strictly increasing recalibration rescales the horizontal axis without reordering the points.}
\label{fig:mechanism}
\end{figure}

On the main-suite ALFWorld rows pooled across runs (Table~\ref{tab:calibration}), Platt scaling cuts the confidence score's calibration error from 0.463 to 0.006 and leaves its control regret at 0.318 after thresholds are reselected; isotonic regression, which creates ties, raises the failure score's regret from 0.358 to 0.462. Training on intervention outcomes changes the order itself; Table~\ref{tab:matched} measures what that recovers.

\subsection{Fewer handoffs, more successes online}
\label{sec:closedloop}\label{sec:closedloop-setup}

The offline suites use an oracle expert and fixed branch outcomes. Does the gain survive a practical intervention, when the controller's own decisions shape the trajectory?

\paragraph{Setup.} The closed-loop suite uses 300 fresh ALFWorld task instances that do not overlap training or validation tasks, with 50 instances in each of six model-seed runs. Controllers are trained on offline data collected with the same GPT-5.4 \citep{openai2026gpt54} handoff intervention and frozen before testing; within each run, policies share task instances and environment seeds. After each base-agent step, the controller chooses continuation, handoff, or quitting; the offline continuation branch instead lets the base agent run to the end of the episode, so branch values are those of committing to an action at $h$. Once handoff occurs, the stronger model completes the episode.

Net utility is the primary control objective. For an episode with $T$ environment steps,
\begin{equation}
U=r_{\mathrm{terminal}}-0.05\,\mathbf{1}\{\mathrm{handoff}\}-0.01\,T,
\label{eq:online-utility}
\end{equation}
where $r_{\mathrm{terminal}}$ is $1$ for success, $-1$ for terminal failure, and $0$ for quitting. The symmetric terminal reward is deliberate: a policy cannot raise utility by handing off indiscriminately (the always-hand-off policy in Figure~\ref{fig:online_budget}). Closed-loop confidence intervals use a paired bootstrap of the six run-level differences with 10,000 resamples.

\paragraph{Results.} Handing off more often does not buy more success here. Failure-triggered routing hands off on 48\% of episodes and succeeds on 37\%. The prefix-feature controller hands off on 35\% and succeeds on 45\%. Handing off on every episode succeeds on 40\%, less often than handing off on a third of them.

\begin{figure}[!htb]
\centering
\includegraphics[width=\linewidth]{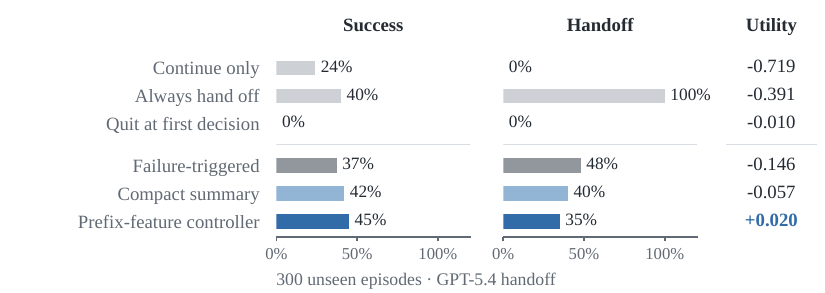}
\caption{\textbf{Higher success with fewer handoffs.} Closed-loop ALFWorld with a fixed GPT-5.4 handoff, pooled over 300 unseen episodes in six runs. Utility follows Eq.~(\ref{eq:online-utility}); quitting at the first decision is an analytical reference. Full results: Table~\ref{tab:closedloop}.}
\label{fig:online_budget}
\end{figure}

Net utility, the control objective, orders the three learned routing pipelines the same way. The prefix-feature controller improves mean utility over failure-triggered routing by 0.166, with a paired-bootstrap 95\% interval of $[0.153, 0.176]$; the compact summary improves it by 0.089 $[0.073, 0.108]$; the prefix-feature controller improves over the compact summary by 0.077 $[0.060, 0.094]$. All three improvements are positive in every run. The utility scale is compressed because a terminal failure costs what a success earns, so the success and handoff rates in Figure~\ref{fig:online_budget} carry the comparison that utility summarizes.

\subsection{The gap appears where both conditions hold}
\label{sec:boundaries}

Section~\ref{sec:regimes} predicts a large gap only when the intervention has value and the summary merges states with different optimal actions. The four benchmarks of Table~\ref{tab:benchmarks} differ in both.

\begin{table}[!htb]
\caption{\textbf{Control gains across four benchmarks.} FT: failure-triggered routing; Full: prefix-feature controller. $\Delta$ is the regret reduction, with paired-bootstrap 95\% intervals over runs; bold reductions exclude zero. Absolute regret is comparable only within a benchmark.}
\label{tab:benchmarks}
\centering\small\renewcommand{\arraystretch}{1.22}
\setlength{\tabcolsep}{4pt}
\begin{tabular*}{\linewidth}{@{\extracolsep{\fill}}llrrr@{}}
\toprule
Benchmark & Intervention & FT & Full & $\Delta$ [95\% CI] \\
\midrule
ALFWorld & Expert handoff & 0.506 & 0.110 & \textbf{0.396} $[0.310, 0.488]$ \\
ScienceWorld & Gold-prefix handoff & 0.245 & 0.169 & \textbf{0.076} $[0.030, 0.130]$ \\
\addlinespace[4pt]
GSM8K & Same-model verify & 0.423 & 0.394 & \textbf{0.029} $[0.006, 0.053]$ \\
HotpotQA & Same-model verify & 0.436 & 0.417 & 0.019 $[-0.020, 0.061]$ \\
\bottomrule
\end{tabular*}
\end{table}

ALFWorld is consistent with both. The expert intervention has strongly positive realized utility on many high-risk prefixes (Figure~\ref{fig:mechanism}), and the gain is positive in all 25 cost cells over the 56 ALFWorld runs, including those with degraded experts (Table~\ref{tab:cost}) and survives a non-oracle GPT-5.4 handoff in the closed-loop suite (Section~\ref{sec:closedloop})\footnote{An earlier version of this paper reported a cross-model repair condition; we removed it because its failure-triggered baseline did not follow the deployed pipeline.}. ScienceWorld shows a smaller but reliable gain under a strong intervention, positive in 23 of 25 cost cells.

The two reasoning benchmarks show small deployed gains. A small gain admits two readings: the score may already separate the actions, or the intervention may be too weak for any router to matter. We probe this by replacing the same-model verify branch with a gold-answer upper bound and recomputing the controllers on the counterfactual data, with the prefixes and the failure score unchanged (Appendix~\ref{app:oracle_probe}). With a fixed controller, GSM8K's gain rises from 0.028 to 0.127, while HotpotQA's moves only from 0.043 to 0.052 (Table~\ref{tab:oracle_probe}). Both deployed gains are small; under the probe they diverge. This is consistent with a weak intervention limiting GSM8K and with the failure score carrying more of the decision on HotpotQA. At the weak end, prompt-only same-model repair completes the task on up to three of eight branches in an eight-prefix ALFWorld pilot but beats continuation on at most one, and every validation-selected controller chose quit (Appendix~\ref{app:prompt_only}). Appendix~\ref{app:ftgap} gives a branched-data diagnostic for where the gap appears; Appendix~\ref{app:transfer} reports cross-model transfer.

\subsection{A quarter of the branches retains 92\% of the gain}
\label{sec:cost}

Every result so far assumes that each prefix was branched on every action, and branching is the price of the supervision. In our profiling, collecting both suffixes costs 2.08 times a continuation-only suffix on ALFWorld and 2.04 times on ScienceWorld. That is a development-time cost; the deployed controller decides in 1.4 ms on a CPU. We measure how much branching the controller needs in a separate ALFWorld suite with four actions: continue, same-model replan, GPT-5.4 handoff, and quit (Figure~\ref{fig:budget_curve}). A quarter of the additional branches, chosen by uncertainty, retains 92\% of the full-data gain; the remaining 0.020 has a 95\% interval of $[0.010, 0.030]$.

\begin{figure}[!htb]
\centering
\includegraphics[width=\linewidth]{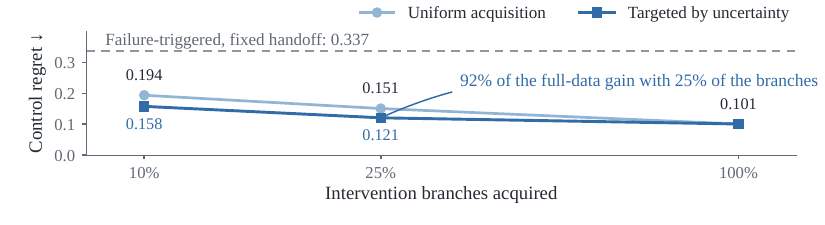}
\caption{\textbf{A quarter of the branches retains 92\% of the gain.} Four-action ALFWorld suite, 12 runs. Targeted acquisition selects prefix-action pairs by uncertainty. This suite has a different action set and oracle from the main suite; full results appear in Table~\ref{tab:budget}.}
\label{fig:budget_curve}
\end{figure}

\section{Related Work}
\label{sec:related}

\paragraph{Routing on continuation-risk scores.}
Confidence and uncertainty estimation provide signals for oversight \citep{kadavath2022language,farquhar2024detecting,xuan2026confidence_dichotomy,zhang2026agentic_calibration}. \citet{ren2023knowno} use calibrated uncertainty to decide when to seek help, and \citet{ding2026calibrate_then_act} use calibrated beliefs for cost-aware exploration. Closest in setting, \citet{wang2026failfast} train a prefix-based failure monitor for SWE agents and pair it with a same-policy restart, showing that an alarm's value depends on the recovery action. \citet{pinto2026calibration_family} examines how a trusted monitor's deferral accuracy depends on the policy lineage it was calibrated on. Our comparison isolates the supervision target with features, estimator, and router fixed, asking whether the resulting signal retains the action comparison that intervention requires.

\paragraph{Deferral and intervention value.}
Learning to defer studies how a predictor should allocate decisions to a downstream expert, with expert performance entering the objective \citep{mozannar2020consistent,madras2018predict,mao2023two_stage,verma2023learning}. This is the closest formal precedent for our comparison; a predicted intervention-success scalar, one of the candidate summaries in Appendix~\ref{app:scalar_upper}, is a deferral-style estimate of expert accuracy. In an agent trajectory that outcome is not observed unless the action is executed, which is why we branch. For fixed intervention and continuation policies, $Q_i(h)-Q_{\cont}(h)$ is a conditional effect on downstream utility, connecting the setup to uplift modeling \citep{gutierrez2017uplift}; \citet{tsiourvas2025causal} learn query-level LLM routing by minimizing regret from observational data, whereas our decisions are taken at trajectory prefixes with executed alternatives. \citet{vasudev2026failure_prevention} establish a prediction-prevention gap for LLM agents. What neither line supplies is the measurement in Table~\ref{tab:matched}: with the pipeline fixed, changing the supervision target from failure to intervention utility removes most of the control loss, and a prefix-feature controller trained on those outcomes keeps its lead over failure-triggered routing online.

\paragraph{Richer signals.}
Selective prediction \citep{geifman2017selective,geifman2019selectivenet}, reasoning-time branch selection \citep{wang2022selfconsistency,yao2023tree_of_thoughts,snell2024scaling}, process-level verifiers \citep{lightman2024lets,uesato2022solving}, post-hoc failure attribution \citep{yeh2026tracing}, and test-time orchestration \citep{qin2026atlas} supply richer signals than a single risk score; our criterion applies to each, since a signal supports an intervention decision to the extent that it preserves the comparison between actions.

\section{Conclusion}
\label{sec:conclusion}

Oversight is a decision, and a failure score answers only half of it. A monitor trained to predict failure tells the controller when continuing is bad, not whether anything better is available. On ALFWorld, changing what the monitor is trained on, from failure to intervention outcomes, accounts for most of the control loss with the pipeline fixed, on two prefix distributions. Recalibrating the failure score cannot do this, because it does not reorder states. A single intervention-side scalar removes most regret with the expert and its noisy variants; under truncation only the two-value predictor still improves on the failure label. The prefix-feature controller's lead over failure-triggered routing survives online with a real stronger-model handoff. Oversight signals should be evaluated by task utility, success, and cost alongside predictive quality.

\paragraph{Limitations and future work.} The in-sample reference of Table~\ref{tab:ladder} bounds two-threshold rules on the learned score, not every rule on it. Online, the prefix-feature controller's utility exceeds the quit reference by 0.030 without an interval, and without a learned continue-or-quit controller or a fixed handoff budget the online gain is not attributed between better stopping and better handoffs. The gold-answer probe changes the intervention and the supervision, and with an intervention that always succeeds the failure probability is itself sufficient, so its added gain measures estimation and routing rather than missing information (Appendix~\ref{app:oracle_probe}). The budget result is confined to the four-action suite, whose failure-triggered row cannot choose replan. Strong results are concentrated on ALFWorld. The next experiments follow from these limits: a matched-target ablation with a real handoff, a learned continue-or-quit controller under a fixed handoff budget, and agents whose recovery actions differ from a handoff, such as restarts for software agents.

\begin{ack}
This research/project is supported by the National Research Foundation, Singapore under its National Large Language Models Funding Initiative (AISG Award No: AISG-NMLP-2024-003). Any opinions, findings and conclusions or recommendations expressed in this material are those of the author(s) and do not reflect the views of National Research Foundation, Singapore.
\end{ack}

\bibliographystyle{plainnat}
\bibliography{references}

\clearpage
\appendix

\section{Proofs}
\label{app:proofs}

\subsection{\texorpdfstring{Proof of \Cref{prop:lossless}}{Proof of Proposition 1}}

\begin{proof}
Fix a measurable summary $g$.

(\emph{If})\;
Suppose $g$ is sufficient in the sense of \Cref{def:suff}, so there exists a measurable selector $d$ such that
$d(g(h)) \in A^*(h)$ almost surely.
Then the $g$-routed controller $\pi_g(h)=d(g(h))$ selects an optimal action almost surely and is therefore lossless.

(\emph{Only if})\;
Suppose no such measurable selector exists.
Then every $g$-routed controller $\pi_g(h)=d(g(h))$ fails to lie in $A^*(h)$ on a set of positive probability, so no controller measurable with respect to $g$ can be lossless.

For the binary action set, let $E=\{h : d(g(h))=\texttt{intervene}\}$ denote the intervention region induced by a $g$-routed selector.
Losslessness requires intervention on every state with $\tauadv(h)>0$, allows either action when $\tauadv(h)=0$, and forbids intervention on states with $\tauadv(h)<0$.
Hence any lossless selector must satisfy
\[
\{\tauadv(h)>0\}\subseteq E\subseteq \{\tauadv(h)\ge 0\}
\qquad\text{almost surely.}
\]
Conversely, any $g$-measurable set $E$ satisfying this inclusion defines a lossless binary selector.
If one further restricts to threshold policies on a scalar $g$, the intervention region must be representable as a threshold set in $g$, which imposes the stated monotonicity condition.
\end{proof}

\subsection{Abstraction loss identity}
\begin{proposition}[Abstraction loss of routing through a summary]
\label{prop:gap}
Let $V^*=\EE[\max_{a\in\cA}Q_a(h)]$ and $V_g=\EE[\max_a \EE[Q_a(h)\mid g(h)]]$. Then
\[
\Gap(g)=V^*-V_g=\EE\Big[\EE\big[\max_a Q_a(h)\,\big|\,g(h)\big]-\max_a \EE\big[Q_a(h)\,\big|\,g(h)\big]\Big]\ge 0 .
\]
\end{proposition}
\begin{corollary}[Zero loss iff sufficient]
\label{cor:gap_zero}
$\Gap(g)=0$ if and only if $g$ is sufficient in the sense of \Cref{def:suff}.
\end{corollary}

\subsection{\texorpdfstring{Proof of \Cref{prop:gap} and \Cref{cor:gap_zero}}{Proof of Proposition 2 and Corollary 1}}

\begin{proof}
The optimal unconstrained controller achieves expected utility
\[
V^* = \EE\!\big[\max_{a \in \cA}\;Q_a(h)\big] = \EE\!\Big[\EE\!\big[\max_{a}\;Q_a(h) \,\big|\, g(h)\big]\Big],
\]
where the second equality follows from the tower property of conditional expectation.

The optimal summary-routed controller, at each value of $g$, must pick the single action with the highest conditional expected utility.
Its value is
\[
V_g = \EE\!\Big[\max_{a}\;\EE\!\big[Q_a(h) \,\big|\, g(h)\big]\Big].
\]

The abstraction loss of the best $g$-routed controller is therefore
\[
\Gap(g) = V^* - V_g
= \EE\!\Big[\EE\!\big[\max_{a}\;Q_a(h) \,\big|\, g(h)\big] - \max_{a}\;\EE\!\big[Q_a(h) \,\big|\, g(h)\big]\Big].
\]
This proves \Cref{prop:gap}.

For \Cref{cor:gap_zero}: if a sufficient selector $d$ exists, then $V_g\ge\EE[Q_{d(g(h))}(h)]=\EE[\max_a Q_a(h)]=V^*$, and $V_g\le V^*$ always, so $\Gap(g)=0$. Conversely, if $\Gap(g)=0$, the integrand $\EE[\max_a Q_a(h)\mid g]-\max_a\EE[Q_a(h)\mid g]$ is nonnegative with zero expectation, hence zero almost surely; taking $d(g)$ as a measurable argmax of $\EE[Q_a(h)\mid g]$ gives $\EE[\max_a Q_a(h)-Q_{d(g(h))}(h)\mid g]=0$, and since the integrand is nonnegative, $Q_{d(g(h))}(h)=\max_a Q_a(h)$ almost surely, so $d(g(h))\in A^*(h)$.
\end{proof}

\subsection{Conditional conflict lower bound}

\begin{proposition}[Conditional conflict lower bound]
\label{prop:conflict}
Let $\cA$ be a finite action set, let $Q_a(h)=\EE[U_a\mid h]$ be integrable, and fix a measurable summary $g(h)$. Suppose $S_a,S_b$ are disjoint measurable state sets, $a\ne b$, and $\gamma>0$ such that, almost surely,
\[
 Q_a(h)\ge Q_c(h)+\gamma\quad(h\in S_a,\ c\ne a),\qquad
 Q_b(h)\ge Q_c(h)+\gamma\quad(h\in S_b,\ c\ne b).
\]
Write $q_a(g)=\EE[\mathbf 1_{S_a}\mid g]$ and $q_b(g)=\EE[\mathbf 1_{S_b}\mid g]$. Every deterministic policy $\pi(h)=d(g(h))$ satisfies
\[
 \EE\!\left[\max_{c\in\cA}Q_c(h)-Q_{\pi(h)}(h)\right]
 \ge \gamma\,\EE\!\left[\min\{q_a(g),q_b(g)\}\right].
\]
\end{proposition}
\begin{proof}
Pointwise regret is at least $\gamma\mathbf 1_{S_a}\mathbf 1_{\{d(g)\ne a\}}+\gamma\mathbf 1_{S_b}\mathbf 1_{\{d(g)\ne b\}}$. Conditional on $g$, its expectation is therefore at least $\gamma\big(q_a(g)\mathbf 1_{\{d(g)\ne a\}}+q_b(g)\mathbf 1_{\{d(g)\ne b\}}\big)$. If $d(g)=a$, this is $\gamma q_b(g)$; if $d(g)=b$, it is $\gamma q_a(g)$; otherwise it is $\gamma(q_a(g)+q_b(g))$. Each case is at least $\gamma\min\{q_a(g),q_b(g)\}$. Taking expectations proves the result.
\end{proof}
The bound conditions on the observed value of $g$, not on an interval containing distinct values; a bin-constant policy is covered by taking the observed summary to be the bin label.

\subsection{\texorpdfstring{Proof of \Cref{rem:recal}}{Proof of Remark 1}}
For strictly increasing $T$, $T(g(h))>T(t)$ holds if and only if $g(h)>t$, so every threshold policy on $g$ is a threshold policy on $T\circ g$ and conversely; thresholds outside the range of $g$ give the two constant policies in both cases. With several thresholds and a fixed action ordering, apply the same identity to each threshold. Continuity is a sufficient assumption covering the Platt transform; boundary conventions need separate care for discontinuous monotone maps.

\section{Simulation validation}
\label{app:simulation}

To examine the theoretical quantities in a controlled setting where the data-generating process is known, we construct a synthetic environment and compute the binned abstraction loss of the failure score on generated samples. For brevity, we denote this appendix-only quantity by $\NMG$.

\subsection{Simulation design}

We generate $n = 50{,}000$ prefix states, each characterized by a latent difficulty and recoverability profile.
The utility distributions for each action are constructed so that the sufficiency condition of \Cref{prop:lossless} is violated by design: two states can have the same $\pfail(h)$ but different optimal actions, because one is \emph{hard but recoverable} (intervention succeeds) while the other is \emph{hard and unrecoverable} (intervention also fails).

\paragraph{Core experiment (binary action).}
The action set is $\{\texttt{continue}, \texttt{intervene}\}$.
$\NMG$ is the abstraction loss $\Gap(b(\pfail))$ of the binned score on the generated sample, a finite-sample quantity for the binned score rather than the population $\Gap(\pfail)$. We compute it by binning states on true $\pfail$ (200~bins) and comparing the oracle utility (best action per state) to the best-per-bin failure-triggered utility (best action per bin):
\[
\NMG_{\text{binned}} = \frac{1}{n}\sum_{i=1}^{n} \max_a Q_a(h_i) - \frac{1}{n}\sum_{i=1}^{n} Q_{a^*_{\text{bin}(i)}}(h_i) = 0.0383.
\]
The action disagreement rate between the oracle and the best failure-triggered rule is $18.1\%$, indicating that nearly one in five states is misclassified by failure-triggered routing.

We also compute the isotonic regression bound, which fits the best \emph{monotone} mapping from $\pfail$ to actions.
This is a failure-threshold controller with a monotone $f$. On this generated sample its binned loss is $0.0395$ against $0.0383$ for the bin-constant policy; the two policy classes are not nested (a monotone threshold can split a bin, and a bin-constant policy can be non-monotone), so no ordering between them is implied.

\paragraph{Three-action experiment.}
We extend to $\{\texttt{continue}, \texttt{verify}, \texttt{quit}\}$ with a catastrophic penalty for incorrect continuation ($c = -1.25$).
The NMG increases to $0.0670$ after changing both the action set and the wrong-answer penalty; the increase is not attributed to the added action alone.
The oracle action distribution is: continue $58.1\%$, verify $20.1\%$, quit $21.8\%$.

\subsection{Key validation}

\paragraph{Learned controller versus binned quantity.}
We train a failure-triggered controller with a learned failure model (logistic regression on noisy observed features, observe noise $\sigma = 0.35$) and measure its regret.
The learned failure-triggered controller achieves regret $0.0388$, close to the binned quantity $0.0383$.
The binned quantity is not a bound on a learner that uses noisy features, since such a score need not be a function of the true-risk bin, and the difference is not attributed to estimation error alone.

\paragraph{Noise sensitivity.}
In this experiment the latent sample and the bins are held fixed and only the observation noise of the learned score changes, so the binned quantity does not move.
We vary the observation noise $\sigma$ from $0$ to $1.0$. The binned quantity stays at $0.0383$ while the learned regret ranges over $0.037$--$0.039$, and at the two highest noise levels it falls below the binned value (0.0379 and 0.0366), so the binned quantity is not a bound for the learned controller.

\section{Protocol and Benchmark Details}
\label{app:benchmarks}\label{app:protocol}

\subsection{Reward structure}

In the offline main suite, each benchmark uses a three-action control set with quit worth zero on the remaining suffix; the four-action suite of \Cref{sec:cost} and the closed-loop suite, which charges costs as they occur, are defined in their own sections.
The utility of action $a$ at prefix state $h$ is:
\[
U_a(h) = \begin{cases}
\text{(outcome score)} - c_a - s \cdot (\text{branch steps}), & a \in \{\texttt{continue}, \texttt{intervene}\}, \\
0, & a = \texttt{quit},
\end{cases}
\]
where $c_a$ is a fixed intervention cost and $s$ is a per-step cost.

\begin{table}[!htbp]
\centering
\caption{Utility parameters per benchmark. Outcome score is $+1$ for correct, $-w$ for incorrect (binary); ScienceWorld uses a continuous environment score $\in [0,1]$.}
\label{tab:rewards}
\small
\begin{tabular}{lcccl}
\toprule
Benchmark & Wrong penalty $w$ & Intervention cost $c_a$ & Step cost $s$ & Intervention \\
\midrule
ALFWorld & $1.0$ & $0.05$ & $0.01$ & Defer to expert \\
ScienceWorld & --- & $0.05$ & $0.01$ & Defer to gold prefix \\
GSM8K & $0.5$ & $0.10$ & --- & Verify (re-answer) \\
HotpotQA & $0.5$ & $0.10$ & --- & Verify (re-answer) \\
\bottomrule
\end{tabular}
\end{table}

For GSM8K and HotpotQA, there are no step costs because these are single-turn reasoning tasks.
For ScienceWorld, the environment returns a continuous score rather than a binary outcome, so there is no separate wrong penalty; instead, utility equals the environment score minus costs.
Quit always yields utility zero.
These asymmetries are intentional: intervention carries a positive cost.

\subsection{Intervention definitions}
\label{app:interventions}

\textbf{ALFWorld.} The defer-to-expert intervention hands the episode to the environment's built-in expert planner, which executes its plan from the branched state until the task is solved or the 20-step episode budget is exhausted; the base agent does not resume. An expert failure therefore means that the planner found no plan from that state or could not finish within the remaining steps, for instance after an irreversible action, not that the agent was interrupted. The intervention-quality ladder degrades this expert in two ways: noise variants replace each expert step with a random admissible non-expert command with probability 0.10, 0.25, or 0.40; truncated variants execute two or three expert steps and stop, and the outcome is the episode state at that point. Costs are in Table~\ref{tab:rewards}.

\textbf{ScienceWorld.} Decision prefixes are states along the environment's gold action sequence for each task variation, two prefixes per episode under a 40-step limit. Continuation lets the base agent act from that state until the episode ends. The gold-prefix handoff executes the next three gold actions from the branched state and then returns control to the base agent for the rest of the episode; the outcome is the environment score in $[0,1]$.

\textbf{GSM8K and HotpotQA.} The same-model verify intervention asks the base agent to re-derive its answer in one turn with tool support, a Python expression on GSM8K and a retrieval query on HotpotQA, and returns the re-derived answer.

\textbf{Closed loop.} The handoff target is GPT-5.4 through the OpenAI API; after a handoff the stronger model completes the episode and the base agent does not resume.

\subsection{Data splits and pooling}
\label{app:splits}

Splits are counted in tasks, with two decision prefixes per task. The original interactive suite used $4/2/4$, $6/3/6$, or $8/4/8$ tasks (train/validation/test) per model--seed combination, corresponding to $8/4/8$, $12/6/12$, or $16/8/16$ prefixes, respectively. The enlarged within-model suite adds $20/10/20$ tasks for later seeds, corresponding to $40/20/40$ prefixes; the 26 ALFWorld runs have 8 to 40 test prefixes each and 568 in total. For the supporting benchmarks (GSM8K, HotpotQA), the split is $48/24/48$ prefixes, one per question.
Two aggregation modes appear in this paper and are named where used: per-run training with run-level aggregation (main suite), and pooled fitting on concatenated rows (marked pooled). In the ALFWorld main suite, a task-identity audit of the submitted runs found no test task in the training or validation split of any run, within or across models and seeds. The 84 distinct held-out tasks are each used by 3.4 runs on average. The audit covers this suite only. For the other suites we rely on the trajectory-level split rule of \Cref{sec:branching}, and the pooled analyses below concatenate rows across runs without a separate task-level audit. The primary within-model suite pools across seeds and both agent models as follows:
\begin{itemize}
    \item ALFWorld: 26 runs (7--8B tier), pooling the original $4/2/4$ to $8/4/8$ suites and the enlarged $20/10/20$ suite.
    \item ScienceWorld: 16 runs in the within-model suite, for $n_{\text{test}} = 364$ pooled prefix states.
    \item GSM8K: 6 runs (3 seeds $\times$ 2 models), $n_{\text{test}} = 288$ pooled prefix states.
    \item HotpotQA: 6 runs (3 seeds $\times$ 2 models), $n_{\text{test}} = 288$ pooled prefix states.
\end{itemize}
The exploitability analysis of \Cref{app:ftgap} also includes the 30 ALFWorld runs with degraded experts (\Cref{app:matched}), giving 56 ALFWorld runs and 84 regimes in total.
Multiple scalar and cost sweeps reuse these same branched trajectories; they add evaluation settings, not environment rollouts.

\subsection{Agent models}
\label{app:agents}

The base agents are Qwen2.5-7B-Instruct~\citep{qwen2025qwen25} and Llama-3.1-8B-Instruct~\citep{dubey2024llama3} on every benchmark.
GPT-5.4~\citep{openai2026gpt54} is used as an intervention, not as a base agent: it is the handoff target in the closed-loop suite (\Cref{sec:closedloop-setup}) and in the four-action suite (\Cref{sec:cost}).
All models are accessed via standard inference without any benchmark-specific fine-tuning.
The controller (random forest) is trained on trajectory-level features extracted from the agent's interaction history.

\subsection{Feature representation}
\label{app:features}

At each prefix state the controller observes a fixed-length feature vector extracted from the agent's trajectory so far.
Our \emph{primary} analyses use only \emph{prefix-available} features for both failure-triggered routing and the prefix-feature controller.
This keeps the information available to the compared controllers symmetric.
Auxiliary preview-aware variants, used only in this appendix, append features of the intervention output. They are reported separately because a preview is information the base comparison does not have.

\paragraph{Feature categories.}
All benchmarks share several common feature families: (i)~\emph{length statistics} of the question/task description, observation, and model response; (ii)~\emph{progress indicators} such as step index and history length; (iii)~the agent's self-reported \emph{confidence} score (normalized to $[0,1]$).
Benchmark-specific features are summarized below.

\begin{itemize}
    \item \textbf{GSM8K.}
    \emph{Continue branch:} question and response lengths, count of numbers in the question, answer-format indicators (has final answer tag, decimal, negative), numeric magnitude features (answer value, max/min/mean of question numbers), range-check flags (answer exceeds question range), hedging-word indicator, and 12 math-keyword indicators (\texttt{each}, \texttt{total}, \texttt{percent}, \texttt{half}, etc.), plus confidence.
    \emph{Verify branch (additional):} verify-expression length, tool-execution success flag, verify answer value, continue--verify agreement, answer gap and gap ratio, range-check flags for both answers, answer-in-question flags, and 6 arithmetic-operator indicators ($+,-,\times,/,(,)$).

    \item \textbf{HotpotQA.}
    \emph{Continue branch:} question and response lengths, answer token count, answer-present flag, answer-in-question overlap, yes/no question and answer indicators, topic-keyword indicators (\texttt{film}, \texttt{person}, \texttt{city}, \texttt{year}, etc.), 7 wh-word indicators, plus confidence.
    \emph{Verify branch (additional):} search-query length, evidence length, retrieval-success flag, top and mean retrieval scores, number of unique retrieved titles, query--question overlap, continue--verify agreement, answer-in-evidence flags, yes/no flags, and answer lengths.

    \item \textbf{ALFWorld.}
    \emph{Continue branch:} step index, history length, task/observation lengths, number of admissible commands, action word count, 10 action-type indicators (\texttt{go~to}, \texttt{open}, \texttt{take}, \texttt{put}, \texttt{heat}, \texttt{cool}, \texttt{clean}, \texttt{toggle}, \texttt{close}, \texttt{look}), ``nothing happens'' flag, 7-dim task-type one-hot, plus confidence.
    \emph{Verify branch (additional):} step index, admissible-command count, verify-action word count, continue--verify agreement, 8 verify action-type indicators, observation state flags (\texttt{closed}, \texttt{open}), and 7-dim task-type one-hot.

    \item \textbf{ScienceWorld.}
    \emph{Continue branch:} step index, task-description length, observation length (words and characters), number of valid actions, action word count, history length, 7 action-type indicators (\texttt{go~to}, \texttt{open}, \texttt{close}, \texttt{pick~up}, \texttt{pour}, \texttt{use}, \texttt{activate}), \texttt{look~around} flag, ``you see'' observation flag, plus confidence.
    \emph{Verify branch (additional):} step index, valid-action count, continue--verify agreement, action-length difference, 2 \texttt{go~to} flags (continue and verify), \texttt{look~around} flag, ``focus on task'' observation flag.
\end{itemize}

\noindent
The prefix-only feature dimensionalities are: GSM8K~($32$), HotpotQA~($19$), ALFWorld~($26$), ScienceWorld~($17$).
For the auxiliary preview-aware variants, the intervention branch appends additional features, yielding GSM8K~($56$ verify), HotpotQA~($35$), ALFWorld~($40$), and ScienceWorld~($25$).

\paragraph{Predictors and routers.}
Failure and correctness predictors use scikit-learn's \texttt{Random\-Forest\-Classifier}; utility and advantage predictors use \texttt{Random\-Forest\-Regressor}. Both use $200$ trees, maximum depth~$6$, and minimum samples per leaf~$2$. The matched-target ablation also fits binary targets with the regressor to hold the estimator fixed (\Cref{app:matched}).
For the conservative controller, the standard deviation of per-tree predictions supplies the uncertainty estimate in the lower confidence bound. Router search spaces: the two-threshold rule searches both thresholds on a 12-point quantile grid of the validation scores; the value router searches 12 margin quantiles and 13 quit floors in $[-0.2, 0.4]$; the lower-confidence router searches $\beta_{\cont}\in\{0,0.1,\dots,1\}$ and $\beta_{i}\in\{0,0.1,\dots,1.5\}$. Each is selected by validation utility.

\paragraph{Feature importance.}
To check that the controller is not relying on opaque feature-engineering artifacts, we refit pooled interactive-benchmark models and inspect grouped random-forest feature importance.
\Cref{fig:feature_importance} should be read as an \emph{auxiliary} preview-analysis figure: failure and continue models rely mostly on trajectory progress and state-shape signals, while preview-aware intervention models place additional weight on intervention-branch features.
The figure separates two questions: the main comparisons ask what the prefix alone achieves, and the auxiliary preview-aware analysis asks how much further preview information helps.

\begin{figure}[!htbp]
    \centering
    \includegraphics[width=0.95\linewidth]{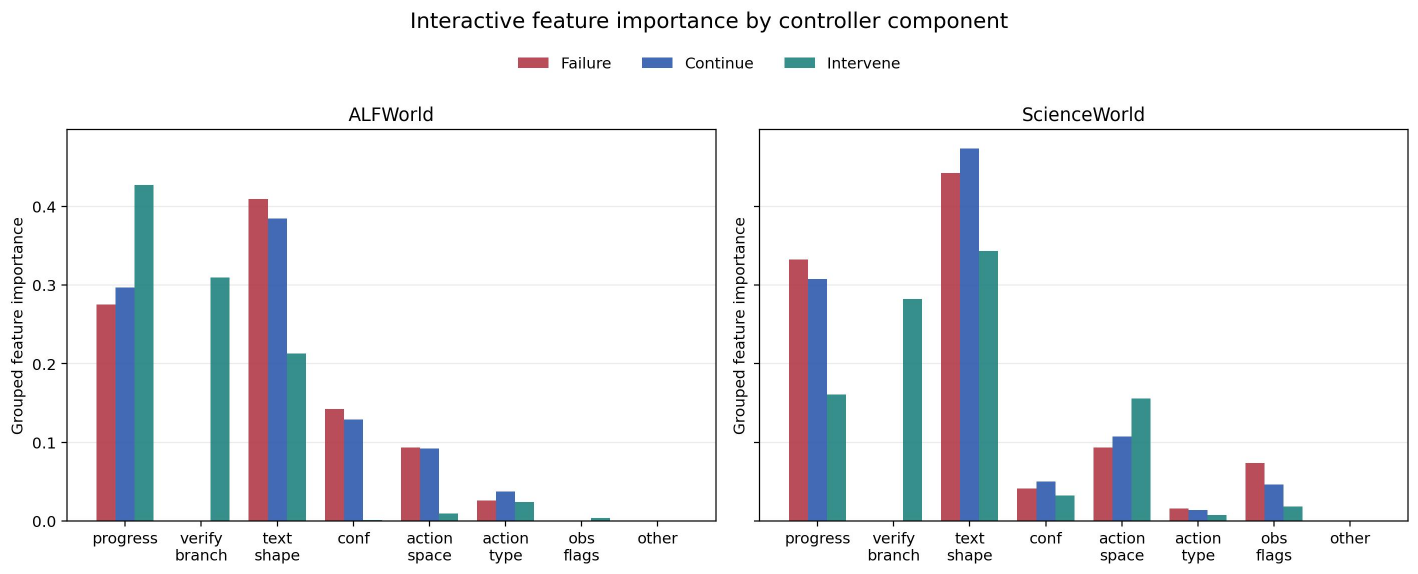}
    \caption{Grouped feature importance for pooled interactive-benchmark models in the auxiliary preview-aware analysis. Failure and continue models depend mostly on trajectory-progress and state-shape signals, whereas preview-aware intervention models also rely on intervention-branch features.}
    \label{fig:feature_importance}
\end{figure}

\subsection{Statistical testing}
\label{app:stats}

Unless stated otherwise, significance tests use paired bootstrap with $10{,}000$ resamples at the \emph{run} level; task-clustered sensitivity intervals are identified explicitly.
Each resample draws run-level differences with replacement and computes the mean difference.
A comparison is marked significant if the $95\%$ confidence interval excludes zero. Differences and intervals are computed before rounding the displayed values.
This choice treats a run, rather than an individual prefix, as the statistical unit and respects dependence among prefixes within a run; test tasks recur across runs, so runs are not independent across tasks. As a sensitivity check we resample the 84 held-out ALFWorld tasks instead of the 26 runs, keeping all rows of a task together: the main regret reduction has a task-clustered 95\% interval of $[0.315, 0.404]$ around a row-weighted estimate of 0.358, against the run-level $[0.310, 0.488]$ around the run-weighted 0.396 of Table~\ref{tab:ladder}; the two estimates differ because the runs with 8 to 16 test prefixes show larger reductions than the runs with 40. Task-clustered intervals for the matched-target differences are given in \Cref{app:matched}.

\subsection{Repeated branching}
\label{app:repeated_branching}
Each branch in the main experiments is executed once. To measure within-prefix sampling noise we collected 72 ALFWorld and 72 ScienceWorld test prefixes, balanced across two base models and four seeds, and independently reran both \texttt{continue} and \texttt{intervene} five times from every prefix, keeping both learned controllers fixed and recomputing branch utilities, oracle actions, and regret from the five-run means. Oracle-action labels agree between the single-rollout and five-run estimates on 69/72 ALFWorld and 66/72 ScienceWorld prefixes, and the sign of intervention advantage on 70/72 and 68/72; disagreements concentrate near ties, and on prefixes with five-run $|\tau(h)|>0.10$ agreement is 69/70 and 66/68. On these 72 prefixes, the regret reduction from failure-triggered routing to the prefix-feature controller changes from 0.396 (0.508 vs 0.112) to 0.395 (0.501 vs 0.106) on ALFWorld and from 0.075 (0.246 vs 0.171) to 0.074 (0.240 vs 0.166) on ScienceWorld.

\subsection{Collection and deployment cost}
Profiled on 100 prefixes per interactive benchmark, collecting both suffixes takes 2.08 times the wall-clock time of a continuation-only suffix on ALFWorld and 2.04 times on ScienceWorld. The deployed controller evaluates learned action values in 1.4 ms per decision on a CPU. Open-weight 7--8B inference was run with local vLLM workers on a server with four NVIDIA RTX A5000 GPUs (24 GB each), an AMD EPYC 7443P CPU (24 cores, 48 hardware threads), and 188 GiB RAM. Each vLLM worker used a single GPU; controller training used CPU-based scikit-learn random forests and was lightweight relative to LLM inference and environment rollouts. GPT-5.4 was accessed through the public OpenAI API. Aggregate token counts and total GPU-hours were not logged.

\subsection{ScienceWorld failure boundary}
\label{app:swboundary}
The ScienceWorld failure label is $y_{\mathrm{fail}}(h)=\mathbf 1[U_{\mathrm{continue}}(h)\le 0]$, the point at which continuing is no better than quitting; on the 364 pooled test prefixes, 132 (36.3\%) satisfy it. Relabeling with $U_{\mathrm{continue}}(h)\le\delta$ for $\delta\in\{-0.10,-0.05,0,0.05,0.10\}$, refitting the failure model, and reselecting its routing rule on validation gives failure-triggered regret 0.253, 0.249, 0.245, 0.242, and 0.239; the prefix-feature controller, which does not use the binary target, stays at 0.169.

\subsection{Closed-loop and budget-suite tables}
\label{app:tables_online}

Tables~\ref{tab:closedloop} and~\ref{tab:budget} give the full numbers behind Figures~\ref{fig:online_budget} and~\ref{fig:budget_curve}.

\begin{table}[!htbp]
\caption{Closed-loop ALFWorld with a fixed stronger-model handoff; rates pooled over 300 episodes (50 per run), utility includes failure, handoff, and step costs. Failure is terminal failure. Quit at first decision is an analytical outside option: the controller acts after each base-agent step, so quitting at the first decision costs one step.}
\label{tab:closedloop}
\centering\small
\begin{tabular}{@{}lrrrrrr@{}}
\toprule
Policy & Utility & Success & Failure & Quit & Handoff & Steps \\
\midrule
Continue only & $-0.719$ & 0.240 & 0.760 & 0.000 & 0.000 & 19.87 \\
Always hand off & $-0.391$ & 0.400 & 0.600 & 0.000 & 1.000 & 14.13 \\
Quit at first decision & $-0.010$ & 0.000 & 0.000 & 1.000 & 0.000 & 1 \\
Failure-triggered routing & $-0.146$ & 0.373 & 0.360 & 0.267 & 0.480 & 13.57 \\
Compact summary & $-0.057$ & 0.420 & 0.307 & 0.273 & 0.400 & 15.08 \\
Prefix-feature controller & \textbf{0.020} & \textbf{0.450} & \textbf{0.267} & 0.283 & 0.350 & 14.63 \\
\bottomrule
\end{tabular}
\end{table}

\begin{table}[!htbp]
\caption{Limited-budget branching in the four-action suite (Qwen2.5-7B and Llama-3.1-8B, six seeds each, 12 runs, 20/10/20 splits, 240 training and 240 test prefixes). Budget is the fraction of the 480 additional intervention branches acquired. Each action is the oracle choice on 31\%, 11\%, 35\%, and 23\% of held-out prefixes. Targeted acquisition adds batches of twelve prefix-action pairs with the largest across-tree variance in predicted utility. Absolute regret is not comparable to Table~\ref{tab:ladder}: the action set, utilities, and oracle differ.}
\label{tab:budget}
\centering\small
\begin{tabular}{@{}lrrrr@{}}
\toprule
Controller & Extra branches & Budget & Test regret & Gain retained \\
\midrule
Failure-triggered, fixed GPT-5.4 handoff & 0 & 0\% & 0.337 & \\
Shared $Q(h,a)$, uniform & 48 & 10\% & 0.194 & 61\% \\
Shared $Q(h,a)$, targeted & 48 & 10\% & 0.158 & 76\% \\
Shared $Q(h,a)$, uniform & 120 & 25\% & 0.151 & 79\% \\
Shared $Q(h,a)$, targeted & 120 & 25\% & 0.121 & 92\% \\
Shared $Q(h,a)$, full branching & 480 & 100\% & 0.101 & 100\% \\
\bottomrule
\end{tabular}
\end{table}

\subsection{Protocol map}
\label{app:protocolmap}
The ALFWorld regret values in this paper come from the following protocols and are comparable only within a protocol. Table~\ref{tab:ladder} and Table~\ref{tab:benchmarks}: the main suite (26 ALFWorld runs, 16 ScienceWorld runs, 6 GSM8K runs, 6 HotpotQA runs), oracle defer-to-expert on ALFWorld, per-run training with validation selection and run-level aggregation; the intervention-aligned rows of Table~\ref{tab:ladder} use the same runs and splits with predictors fit per run. Table~\ref{tab:matched}, main-suite columns: the same runs and splits, one estimator and one router for every target (\Cref{app:matched}); mid-episode columns and Table~\ref{tab:midprefix}: the mid-episode prefix suite, eight separate runs with six prefixes per trajectory (\Cref{app:midprefix}). Table~\ref{tab:closedloop}: 300 fresh ALFWorld tasks, six model-seed runs, GPT-5.4 handoff, frozen controllers. Table~\ref{tab:budget}: the separate four-action ALFWorld suite (12 runs, 20/10/20 splits). Table~\ref{tab:calibration}: the 26 main-suite runs with rows pooled across runs. Table~\ref{tab:cost}: the cost-sensitivity sweeps of \Cref{app:cost}. \Cref{app:repeated_branching}: 72 test prefixes per interactive benchmark. Tables~\ref{tab:multiscalar}, \ref{tab:per_variant}, \ref{tab:valsel}, and \ref{tab:transfer}: pooled protocols that concatenate rows across runs, as stated in each caption.

\section{Additional Results}
\label{app:extended}\label{app:ablations}

\subsection{Matched-target ablation}
\label{app:matched}

The main-suite columns of Table~\ref{tab:matched} use the 26 ALFWorld main-suite runs and their splits; the mid-episode columns are described in \Cref{app:midprefix}. For every target the input is the base prefix feature vector with self-reported confidence (26 dimensions), the estimator is the random-forest regressor of \Cref{app:features} (200 trees, depth 6, minimum leaf 2, fixed random seed) with binary labels fit as 0/1 targets, and the router is a two-threshold rule on the predicted scalar. Thresholds are taken from the same 12-point quantile grid as the deployed failure pipeline; thresholds and the assignment of the three actions to the three regions are selected together on the run's validation split by validation utility, so every target may use any action order. Regret is evaluated on the run's test rows and averaged over runs. The two-quantity row fits the continuation and intervention utility regressors and routes with the value rule of the main pipeline (margin and quit floor selected on validation). Its regret of 0.069 is below the 0.110 of Table~\ref{tab:ladder} with the same features because Table~\ref{tab:ladder} selects the controller family on validation splits of 4 to 20 prefixes, whereas this row fixes the value family; \Cref{app:valsel} reports the same pattern for a fixed conservative controller.

The deployed failure pipeline fits the failure label with a classifier and fixes the order continue, intervene, quit; it reproduces 0.506 on the main-suite rows and 0.286 on the mid-episode suite. Fitting the two binary labels with that classifier instead of the regressor gives 0.494 for the failure label (0.500 and 0.525 under two further seeds) and 0.172 for intervention success (0.131 and 0.146), with the utility rows unchanged, run-level intervals $[0.235, 0.407]$, $[0.253, 0.434]$, $[0.299, 0.512]$, and $[0.352, 0.503]$ for the four intervention-side rows, and 23 to 25 of 26 runs improved. Under two further seeds of the regressor version the means are 0.507 and 0.519 for the failure label, 0.109 and 0.089 for intervention success, 0.144 and 0.157 for the advantage, 0.089 and 0.095 for intervention utility, and 0.074 and 0.062 for the two-quantity router. The selected action orders are informative: the failure label keeps the deployed order in most runs, while every intervention-side target places quitting in the low-score region, so a low predicted intervention value routes to quit rather than to continuation. Task-clustered 95\% intervals (\Cref{app:stats}) for the regressor version are $[0.339, 0.424]$ for intervention success, $[0.308, 0.395]$ for the advantage, $[0.368, 0.443]$ for intervention utility, and $[0.367, 0.448]$ for the two-quantity router.

\Cref{tab:matched_ladder} repeats the ablation on the 30 degraded-expert runs of the intervention-quality ladder (six runs per variant, 16 test prefixes each, so individual cells are noisy). With a noisy expert the picture of Table~\ref{tab:matched} persists at every noise level: all four intervention-side rows stay well below the failure label, and their paired intervals over the six runs all exclude zero; the weakest case is intervention utility at noise 0.40, $[0.017, 0.466]$ with four of six runs improved. With a truncated expert that executes two or three steps and stops, failure routing itself has lower regret, the single intervention-side scalars are worse than the failure label on average with intervals covering zero, and only the two-value pair still improves on it, by 0.126 $[0.033, 0.220]$ (four of six runs) and 0.133 $[0.018, 0.239]$ (three of six). Truncation changes both how long the intervention acts and how the episode ends, so it is not only a weaker expert; in these runs a single intervention-side scalar removes most of the regret while the expert is capable, and only the pair still improves on the failure label under truncation.

\begin{table}[!htbp]
\centering
\caption{Matched-target ablation across degraded experts on ALFWorld (same protocol as Table~\ref{tab:matched} with the binary labels fit by a classifier; run-level mean regret; six runs per variant). Noise variants replace each expert step by a random admissible non-expert command with the stated probability; truncated variants execute two or three expert steps and stop.}
\label{tab:matched_ladder}
\small
\begin{tabular}{@{}lcccccc@{}}
\toprule
Expert variant & Failure & Cont.\ util. & Interv.\ succ. & Advantage & Interv.\ util. & Both \\
\midrule
Noise 0.10 & 0.609 & 0.383 & 0.283 & 0.152 & 0.124 & 0.099 \\
Noise 0.25 & 0.605 & 0.380 & 0.174 & 0.133 & 0.058 & 0.099 \\
Noise 0.40 & 0.597 & 0.380 & 0.185 & 0.222 & 0.350 & 0.126 \\
Truncated, two steps & 0.232 & 0.365 & 0.358 & 0.304 & 0.316 & 0.106 \\
Truncated, three steps & 0.347 & 0.366 & 0.380 & 0.381 & 0.434 & 0.214 \\
\midrule
All 30 runs & 0.478 & 0.375 & 0.276 & 0.238 & 0.256 & 0.129 \\
\bottomrule
\end{tabular}
\end{table}
\subsection{Mid-episode prefix suite}
\label{app:midprefix}

To test the same comparisons at decision points spread over the episode, we collected a second ALFWorld suite with the same base agents, task types, split rule, expert intervention, costs, and 20-step budget, but with six decision prefixes per base trajectory at evenly spaced steps (steps 0, 4, 8, 11, 15, and 19 of a 20-step trajectory): eight runs (Qwen2.5-7B and Llama-3.1-8B, seeds 13, 17, 29, and 37), 20/10/20 tasks per run, which gives 115 to 120 training, 57 to 60 validation, and 116 to 120 test prefixes per run against 4 to 20 validation prefixes in the main suite; 952 held-out test prefixes over 65 distinct tasks in total. Mid-episode the realized-best action is intervene on 78\%, quit on 13\%, and continue on 9\% of prefixes, so intervening at every prefix is a strong constant reference. \Cref{tab:midprefix} reports the deployed pipelines and the matched-target ablation under the protocol of Table~\ref{tab:matched}; the matched rows also form the right half of Table~\ref{tab:matched}. Failure-triggered routing does no better than the constant reference. Intervention success, intervention utility, and the two-value pair improve on the failure label in all eight runs; the advantage improves in seven and is the weakest of them, as in the main suite. The regret reduction of the validation-selected controller over failure-triggered routing is 0.199, with a run-level interval of $[0.171, 0.225]$ and a task-clustered interval of $[0.162, 0.228]$. With the binary labels fit by a classifier instead, the matched rows are 0.275, 0.245, 0.087, 0.148, 0.098, and 0.095. Each run took 4 to 17 hours to collect, because replaying to mid-episode prefixes dominates the cost.

\begin{table}[!htbp]
\centering
\caption{Mid-episode prefix suite on ALFWorld (eight runs, six prefixes per trajectory; run-level mean regret; intervals are paired bootstraps over runs). Matched rows use the protocol of Table~\ref{tab:matched}: same features, regressor, and validation-selected two-threshold router; only the label changes.}
\label{tab:midprefix}
\footnotesize
\begin{tabular}{@{}lccc@{}}
\toprule
Policy or supervision target & Regret & Reduction [95\% CI] & Runs \\
\midrule
Failure-triggered routing (deployed pipeline) & 0.286 & --- & --- \\
Intervene at every prefix & 0.263 & 0.023 & --- \\
Action-conditioned controller, validation-selected family & 0.088 & 0.199 $[0.171, 0.225]$ & 8/8 \\
\midrule
Matched: failure label & 0.295 & --- & --- \\
Matched: continuation utility & 0.245 & 0.049 $[-0.005, 0.110]$ & 5/8 \\
Matched: intervention success & 0.081 & 0.214 $[0.178, 0.257]$ & 8/8 \\
Matched: advantage $U_i-U_{\cont}$ & 0.148 & 0.147 $[0.077, 0.222]$ & 7/8 \\
Matched: intervention utility & 0.098 & 0.196 $[0.165, 0.241]$ & 8/8 \\
Matched: both values, pair router & 0.095 & 0.199 $[0.152, 0.255]$ & 8/8 \\
\bottomrule
\end{tabular}
\end{table}

\subsection{How far can better scalars go?}
\label{app:scalar_upper}

Recalibrating the same scalar changes its values but not the threshold rules it can express (\Cref{rem:recal}; measured in \Cref{app:calibration}).
A different question is whether choosing a \emph{different} scalar, one more aligned with intervention advantage, would close the gap without requiring action-conditioned control.
Within each run, we select a scalar and controller variant by validation utility from eight candidates: failure score, self-reported confidence, predicted continuation and intervention correctness, and the expected-utility and direct-value predictions for each action (see \Cref{app:features}). We then average held-out test regret across the 26 ALFWorld runs. This auxiliary search uses the action-conditioned controller family, a broader router search than the single-scalar pipeline in \Cref{tab:ladder}.

Under this per-run protocol, the validation-selected single scalar reaches $0.152$ against $0.506$ for failure-triggered routing and $0.110$ for the prefix-feature controller. Intervention correctness is selected in 16 runs, continuation value in five, intervention value in three, confidence in one, and failure score in one.

Intervention-aligned candidates recover much of the gain, with a remaining mean regret gap of $0.043$ to the prefix-feature controller. Selection varies across runs; the matched comparison in \Cref{app:matched} isolates the effect of the supervision target.

\subsection{Compact multi-scalar routing}
\label{app:multiscalar}

A complementary question is how well a small combination of scalars supports control. This subsection uses pooled fitting; \Cref{app:scalar_upper} and the compact-summary row of Table~\ref{tab:ladder} use per-run fitting and selection. Their regret values are not directly comparable.

\paragraph{Protocol.}
We pool train, validation, and test rows separately across all runs within each benchmark.
For each run, per-row scalar features are computed using models fit on that run's training split; the resulting scalar vectors are concatenated across runs.
We then train multi-class oracle-action classifiers on all $\binom{8}{2} + \binom{8}{3} = 84$ two- and three-scalar combinations from the eight candidate scalars, using both logistic regression (with balanced class weights) and random forest (200 trees, max depth 6, balanced class weights).
The best combination is selected on validation utility and evaluated on held-out test data.
Because this searches over many combinations, the result is an exploratory validation-selected outcome, not a bound.

\paragraph{Note on train-set scalars.}
The scalar features for the training split are in-sample predictions from the per-run random-forest models (the same models that produce the val/test scalars).
This creates a distributional mismatch between train and val/test features whose direction of effect has not been established; we did not run a cross-fit comparison.

\begin{table}[!htbp]
\centering
\caption{Compact multi-scalar routing under the pooled protocol. Gap closed is N/A when the prefix-feature controller does not improve on the baseline. ``Single failure'' is logistic regression on failure score alone; ``best multi-scalar'' is the val-selected combination. ``Gap closed'' is computed relative to the single-failure $\to$ prefix-feature-controller gap within this protocol. All numbers are pooled test regret (lower is better).}
\label{tab:multiscalar}
\small
\begin{tabular}{lcccc}
\toprule
Benchmark & Single failure & Best multi-scalar & Prefix-feature controller & Gap closed \\
\midrule
ALFWorld & $0.451$ & $0.057$ & $0.015$ & $90\%$ \\
ScienceWorld & $0.255$ & $0.217$ & $0.149$ & $36\%$ \\
GSM8K & $0.469$ & $0.476$ & $0.369$ & $-8\%$ \\
HotpotQA & $0.434$ & $0.416$ & $0.442$ & N/A \\
\bottomrule
\end{tabular}
\end{table}

On ALFWorld, the best three-scalar combination (failure score, predicted intervention correctness, and predicted intervention expected utility) reduces pooled regret from $0.451$ for single-failure-score routing to $0.057$, compared with $0.015$ for the prefix-feature controller, closing about $90\%$ of the gap under this pooled protocol. Logistic regression ($0.067$) and random forest ($0.057$) reach similar test regret with the validation-selected combination; this compares two routers and does not isolate the target. The best-performing ALFWorld combinations all include predicted intervention correctness or predicted intervention expected utility, and the failure-score pipeline has higher regret in this comparison.
On ScienceWorld the compact summary closes $36\%$ of the gap, consistent with its smaller gain in \Cref{sec:boundaries}. On GSM8K the compact summary does not improve under the deployed verify intervention; on HotpotQA it improves slightly on the single-failure baseline while the prefix-feature controller does not, so no positive gap is defined (\Cref{tab:multiscalar}). The oracle probe in \Cref{app:oracle_probe} is consistent with weak intervention quality limiting GSM8K and with the failure score carrying more of the decision on HotpotQA.
This supports the paper's reading: the deficit lies in summaries that track failure risk and omit intervention value, not in the use of a single scalar as such.

\subsection{Calibration decomposition}
\label{app:calibration}
Recalibrating a fixed score with Platt scaling and reselecting thresholds on validation improves prediction metrics while the measured control regret of threshold routing is unchanged; isotonic regression creates ties and, on the failure score, worsens both prediction metrics and regret. \Cref{rem:recal} concerns the policy class, not the finite-sample reselection; the equality is a measurement. Values are computed on the main-suite rows pooled across runs and are not comparable to the per-run results of \Cref{tab:ladder}.

\begin{table}[!htbp]
\centering
\caption{Calibration decomposition on ALFWorld, evaluated on the main-suite rows pooled across runs, with thresholds reselected on validation after each transform. Platt scaling improves ECE and Brier while measured control regret is unchanged; isotonic regression on the failure score worsens all three.}
\label{tab:calibration}
\small
\begin{tabular}{llccc}
\toprule
Scalar & Method & ECE & Brier & Control regret \\
\midrule
\multirow{3}{*}{Confidence} & Raw & $0.463$ & $0.397$ & $0.318$ \\
 & Platt & $0.006$ & $0.110$ & $0.318$ \\
 & Isotonic & $0.022$ & $0.110$ & $0.318$ \\
\midrule
\multirow{3}{*}{Failure score} & Raw & $0.060$ & $0.096$ & $0.358$ \\
 & Platt & $0.047$ & $0.094$ & $0.358$ \\
 & Isotonic & $0.097$ & $0.103$ & $0.462$ \\
\bottomrule
\end{tabular}
\end{table}

\subsection{Prompt-only same-model repair probes}
\label{app:prompt_only}

\Cref{tab:realistic} lists the ALFWorld prompt-only same-model repair probes summarized in \Cref{sec:boundaries}. They rarely improve over continuation, and every validation-selected controller chose quit.

\begin{table}[!htbp]
\centering
\caption{ALFWorld prompt-only same-model repair probes. In \emph{reflect and retry}, the base model reflects on its proposed action and chooses a different one; in \emph{replan}, it replans the next three steps before the episode continues; \emph{plan, then replan} first writes a plan. These interventions are boundary evidence, not main positive results: they rarely improve over continuation, and deployable controllers collapse to \texttt{quit}.}
\label{tab:realistic}
\small\setlength{\tabcolsep}{3.5pt}
\begin{tabular}{lccccc}
\toprule
Intervention & Rows & Repair succeeds & Beats continuation & Oracle repair & Oracle quit \\
\midrule
Reflect and retry (Qwen) & $8$ & $1$ & $0$ & $0.000$ & $0.750$ \\
Reflect and retry (Llama) & $8$ & $2$ & $0$ & $0.000$ & $0.750$ \\
Replan, 3 steps (Qwen) & $8$ & $3$ & $1$ & $0.125$ & $0.625$ \\
Plan, then replan, 3 steps (Qwen) & $8$ & $2$ & $0$ & $0.000$ & $0.750$ \\
\bottomrule
\end{tabular}
\end{table}

\subsection{Oracle intervention probe}
\label{app:oracle_probe}

The main text uses an oracle intervention probe on GSM8K and HotpotQA; because the probe changes the intervention and the supervision and retrains the learners, its results are read as consistent with, not proof of, the two readings.
The protocol is as follows: for each existing run, we replace every verify-branch outcome with that of a gold-answer oracle that always succeeds, so the verify branch is correct and its utility is $1-c_a$ at every prefix.
Continue-branch outcomes, features, and quit utilities are unchanged.
We then retrain both the failure-triggered controller (a two-threshold rule on the failure score, thresholds and action order selected on validation) and the EU-LCB controller on this counterfactual data and evaluate on held-out test rows.

\begin{table}[!htbp]
\centering
\caption{Oracle intervention probe on the reasoning benchmarks. ``Gain'' is failure-triggered regret minus EU-LCB regret; higher gain means more room for action-conditioned control. Under oracle intervention, GSM8K's gain increases substantially, while HotpotQA's gain barely changes.}
\label{tab:oracle_probe}
\small
\begin{tabular}{llccc}
\toprule
 & & Failure regret & EU-LCB regret & Gain \\
\midrule
\multirow{2}{*}{GSM8K} & Original & $0.414$ & $0.386$ & $0.028$ \\
 & Oracle & $0.144$ & $0.018$ & $0.127$ \\
\midrule
\multirow{2}{*}{HotpotQA} & Original & $0.454$ & $0.411$ & $0.043$ \\
 & Oracle & $0.072$ & $0.020$ & $0.052$ \\
\bottomrule
\end{tabular}
\end{table}

The probe is a diagnostic, not a deployment result: it asks how the gain changes when intervention quality is removed as a constraint. It cannot reveal a sufficiency gap. With the verify branch always correct, $u_{\texttt{verify}}=0.9$ for every row, and with wrong penalty $0.5$ the continuation value at true failure probability $p$ is $1-1.5p$, so the optimal rule is to intervene whenever $p>1/15$ and never to quit: the true failure probability is sufficient under the probe. The gain that remains between the two learned pipelines therefore reflects how well each estimates and routes on that probability from prefix features, not information the failure score lacks.
On GSM8K, the $4.5\times$ increase in gain is consistent with the deployed same-model verify intervention limiting the achievable gain; the probe changes the intervention and the supervision, so it does not measure a latent $\Gap(g)$.
On HotpotQA, the near-unchanged gain ($1.2\times$) is consistent with the failure score carrying more of the decision in this regime.

\subsection{Cost sensitivity on the interactive benchmarks}
\label{app:cost}
\Cref{tab:cost} reports the deployed gain under a $5\times5$ sweep over intervention cost and wrong-answer penalty, pooling the 56 ALFWorld runs of the main suite and the intervention-quality ladder and the 16 ScienceWorld runs. ALFWorld stays positive in all 25 cells and ScienceWorld in 23.

\begin{table}[!htbp]
\centering
\caption{Cost-sensitivity summary on the interactive benchmarks. Gains are failure-triggered regret minus prefix-only action-conditioned regret; positive values favor richer control.}
\label{tab:cost}
\small
\begin{tabular}{lccc}
\toprule
Benchmark & Positive cells & Gain range & Gain along the cost line ($c$: 0 to 0.20) \\
\midrule
ALFWorld & $25/25$ & $[0.170, 0.345]$ & $0.320 \rightarrow 0.252$ \\
ScienceWorld & $23/25$ & $[-0.008, 0.141]$ & $0.106 \rightarrow -0.006$ \\
\bottomrule
\end{tabular}
\end{table}

\subsection{A branched-data diagnostic for anticipating useful regimes}
\label{app:ftgap}

\Cref{sec:regimes} states the conditions under which richer control pays off.
Can these conditions be measured before committing to a more complex controller?
We propose a development-time diagnostic computed from branched data.
The idea is simple: if a classifier trained on prefix features can predict the oracle action better than a classifier trained on the scalar alone, then the prefix carries control-relevant information that the scalar misses.
We call this quantity \emph{exploitability beyond scalar}. If $\widehat{A}^*_{g}$ denotes the oracle-action predictor trained on scalar $g$ alone and $\widehat{A}^*_{g,x}$ the predictor trained on scalar plus prefix features $x(h)$, then
\[
    \mathrm{Exploit}(g)
    =
    \mathrm{Score}\!\left(\widehat{A}^*_{g,x}\right)
    -
    \mathrm{Score}\!\left(\widehat{A}^*_{g}\right),
\]
where $\mathrm{Score}(\cdot)$ is measured out of sample on held-out prefixes.
In the main experiments we use utility-weighted held-out prediction quality, so higher exploitability means that prefix information beyond the scalar makes oracle-action prediction easier.

The main empirical pattern appears in \Cref{fig:master}.
Across $84$ regimes, exploitability correlates with deployable gain (overall Pearson $r = 0.716$), and the relationship remains visible within ALFWorld itself ($r = 0.604$).
The benchmark means follow the same ordering as \Cref{tab:benchmarks}:
ALFWorld has the largest exploitability ($0.310$) and the largest gain ($0.327$);
ScienceWorld is much lower on both ($0.032$, $0.026$);
GSM8K and HotpotQA sit near the low-gain end.

\paragraph{Interpretation and scope.}
Exploitability shares the same prefix feature space as the downstream prefix-feature controller, so part of the correlation is mechanical: both quantities reflect how much control-relevant information the prefix carries beyond the scalar.
We therefore read exploitability as a development-time diagnostic, a check from branched evaluation data on whether richer control is likely to help in a regime, and not as a model-agnostic causal predictor.
To test whether the association is merely mechanical, we evaluate cross-model transfer: compute exploitability on one agent model's data and measure gain on a different agent model's data, using $13$ matched ALFWorld seed/split pairs of Qwen2.5-7B and Llama-3.1-8B.
Exploitability computed on Qwen2.5-7B data predicts Llama-3.1-8B gains ($r = 0.752$), and Llama-3.1-8B exploitability predicts Qwen2.5-7B gains ($r = 0.563$), both comparable to within-model correlations ($0.761$ and $0.851$).
The association is observed across model families; shared tasks, labels, estimators, and selection are not excluded as sources of the correlation, and this is not a prospective held-out-family validation.
To check that the association is not unique to $g=\pfail$, we repeated the analysis with self-reported confidence as the candidate scalar; the correlation drops to $r = 0.339$ overall and to $r = 0.087$ within ALFWorld (\Cref{tab:scalar_candidates}).
Dependence-robust checks remain positive (\Cref{app:dependence}).

\begin{figure}[!htbp]
    \centering
    \includegraphics[width=0.95\linewidth]{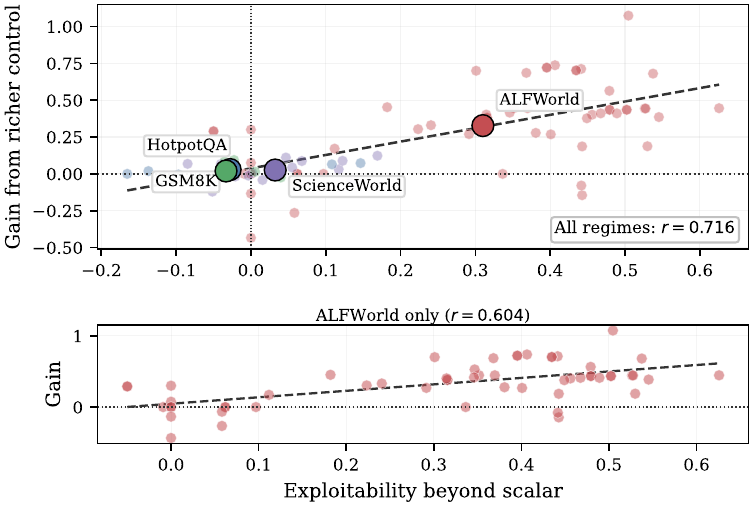}
    \caption{Exploitability beyond scalar predicts the gain from richer control. Each point is one of $84$ regimes. The horizontal axis is the held-out predictive improvement from scalar-only to scalar-plus-prefix prediction of the oracle action; the vertical axis is the regret reduction of the prefix-feature controller over failure-triggered routing. The lower panel shows the same relationship within ALFWorld alone.}
    \label{fig:master}
\end{figure}

\subsection{Dependence-robust view of the exploitability diagnostic}
\label{app:dependence}

Figure~\ref{fig:master} uses one point per completed run/intervention setting from the main suite together with the degraded-expert runs.
It does \emph{not} include the cost-sweep cells or alternate-scalar variants as additional points; those analyses are reported separately in \Cref{sec:boundaries} and \Cref{tab:scalar_candidates}.
To check sensitivity to grouping, we aggregated the $84$ run-level points into $12$ protocol groups: six ALFWorld intervention variants, one ScienceWorld suite, two GSM8K base-model groups, and three HotpotQA seed groups. We then recomputed the correlation between exploitability and deployable gain.
The group-level correlation remains strong ($r = 0.962$).
Leaving out one group at a time keeps the run-level correlation in the range $0.677$--$0.735$.
These checks do not remove all dependence concerns, but they show that the positive association in Figure~\ref{fig:master} is not driven by any single group or by the cost/multi-scalar sweeps.

\subsection{Alternative scalar candidates}

\Cref{tab:scalar_candidates} compares the main failure-score scalar to self-reported confidence on the same $84$ regimes used in the main exploitability analysis.
Both signals are informative; the diagnostic built on the failure score correlates more strongly with gain on these regimes, which does not by itself rank the two scores' control value.
This is consistent with the paper's main thesis.
The question is not whether one can find a scalar that correlates with success at all, but whether a candidate scalar preserves enough information for intervention control.

\begin{table}[!htbp]
\centering
\caption{Exploitability diagnostics for two candidate scalars on the same $84$ regimes. Higher correlation means the measured mismatch better predicts deployable action-conditioned gain.}
\label{tab:scalar_candidates}
\small
\begin{tabular}{lcc}
\toprule
Candidate scalar & Overall correlation & ALFWorld correlation \\
\midrule
Failure score $\pfail(h)$ & $0.716$ & $0.604$ \\
Self-reported confidence & $0.339$ & $0.087$ \\
\bottomrule
\end{tabular}
\end{table}

\subsection{Per-variant action-conditioned estimator results}
\label{app:per_variant}

\Cref{tab:per_variant} reports control regret for several action-conditioned estimator variants across benchmarks.
These tables answer a different question from \Cref{tab:benchmarks}.
\Cref{tab:benchmarks} is the main result, in which failure-triggered routing and the prefix-feature controller see the same prefix-only information.
By contrast, the tables below summarize supplementary experiments that compare estimator variants on the pooled benchmark suites, including auxiliary preview-aware variants used only for supplementary analysis.
Their absolute regret values should therefore be read as \emph{within-appendix estimator ablations}, not as directly comparable replacements for the prefix-only numbers in \Cref{tab:benchmarks}; in particular, the pooled ALFWorld failure-triggered regret here ($0.491$) comes from the pooled data of these supplementary experiments rather than from the main-suite runs behind Table~\ref{tab:benchmarks} ($0.506$).
The \emph{value} variant directly predicts per-action utility, \emph{explicit-EU} decomposes estimation into correctness probability and known reward structure, \emph{EU-LCB} adds a lower confidence bound to explicit-EU, and \emph{two-stage} first decides whether to intervene then selects among interventions.

\begin{table}[!htbp]
\centering
\caption{Control regret by action-conditioned estimator variant in the supplementary estimator comparison (pooled across seeds and models; lower is better). The best variant per benchmark is bolded. These values are within-appendix ablations and are not directly comparable to the prefix-only main results in \Cref{tab:benchmarks}.}
\label{tab:per_variant}
\small
\begin{tabular}{lcccc}
\toprule
Variant & GSM8K & HotpotQA & ALFWorld & ScienceWorld \\
\midrule
Failure-triggered & $0.423$ & $0.436$ & $0.491$ & $0.245$ \\
\midrule
Value (direct utility) & $0.377$ & $0.424$ & $0.050$ & $\mathbf{0.178}$ \\
Explicit-EU ($\beta = 0$) & $\mathbf{0.364}$ & $0.404$ & $\mathbf{0.010}$ & $0.206$ \\
EU-LCB & $0.395$ & $\mathbf{0.358}$ & $\mathbf{0.010}$ & $0.182$ \\
Two-stage & $0.407$ & $0.403$ & $0.146$ & $0.191$ \\
\bottomrule
\end{tabular}
\end{table}

Several patterns emerge.
First, on ALFWorld all action-conditioned variants outperform failure-triggered routing under this supplementary protocol, which includes auxiliary preview-aware variants; this does not identify the source of the main prefix-only gain.
Second, the best variant is benchmark-dependent: explicit-EU dominates on GSM8K, EU-LCB is strongest on HotpotQA, explicit-EU and EU-LCB are tied on ALFWorld at the displayed precision, and the direct value estimator is best on ScienceWorld.
The ordering among estimators is reported as observed; differences in means do not identify bias or variance contributions.
Third, ScienceWorld again behaves like a boundary benchmark: value, EU-LCB, and two-stage all remain competitive with failure-triggered routing, whereas plain explicit-EU is slightly worse.
Finally, the two-stage variant is much weaker than direct action-value estimation on ALFWorld ($0.146$ vs.\ $0.010$) in this suite.

\subsection{Fixed conservative controller versus validation-selected family member}
\label{app:valsel}

The main results (\Cref{tab:ladder}) select among controller families on validation, per run. This section compares that choice with a fixed conservative controller that fixes the family and tunes only its conservatism parameter $\beta$ on validation, using the pooled data of the supplementary estimator comparison (\Cref{app:per_variant}). \Cref{tab:valsel} compares the two.

\begin{table}[!htbp]
\centering
\caption{Fixed conservative controller versus validation-selected best action-conditioned family member in the supplementary estimator comparison (control regret, pooled). Val-selected picks the family member with highest validation utility per run, then evaluates on test. These numbers should be read relative to \Cref{tab:per_variant}, not as direct replacements for \Cref{tab:benchmarks}.}
\label{tab:valsel}
\small
\begin{tabular}{lcccc}
\toprule
Controller & GSM8K & HotpotQA & ALFWorld & ScienceWorld \\
\midrule
EU-LCB & $0.395$ & $0.358$ & $0.010$ & $0.182$ \\
Val-selected & $0.405$ & $0.397$ & $0.063$ & $0.184$ \\
\bottomrule
\end{tabular}
\end{table}

The fixed conservative controller outperforms the validation-selected approach on all four benchmarks, though the margin on ScienceWorld is negligible ($0.182$ vs.\ $0.184$).
This is consistent with sensitivity to validation selection, which the fixed recipe avoids by tuning a single parameter.

\subsection{WebShop diagnostic}

WebShop \citep{yao2022webshop} is included as a supporting diagnostic only, not as a main benchmark, because it has fewer runs ($n = 2$--$3$) and different action semantics.
On WebShop, the measured diagnostic is small ($0.033$--$0.050$); a weaker violation of the sufficiency condition is one possible explanation.
Consistent with this, the validation-selected action-conditioned family member outperforms failure-triggered routing, but the fixed conservative controller does not always improve over failure-triggered routing on this benchmark.
A bias-variance explanation is possible but not identified by these runs.

\subsection{Cross-model transfer}
\label{app:transfer}

In the main experiments, controllers are trained and evaluated on trajectories from the same agent model (within-model regime).
A natural question is whether controllers trained on one model's trajectories generalize to another.
We test this by training on Qwen2.5-7B trajectories and evaluating on Llama-3.1-8B trajectories (and vice versa), across all available seed pairings.

\begin{table}[!htbp]
\centering
\caption{Cross-model transfer versus matched-subset within-model control regret (pooled; lower is better). Transfer pairs train on one model and evaluate on the other, so the within-model column reports the corresponding matched within-model subset rather than the enlarged pooled suite used in the main text.}
\label{tab:transfer}
\small
\begin{tabular}{lcccc}
\toprule
 & \multicolumn{2}{c}{Within-model} & \multicolumn{2}{c}{Cross-model transfer} \\
\cmidrule(lr){2-3}\cmidrule(lr){4-5}
Benchmark & Failure & EU-LCB & Failure & EU-LCB \\
\midrule
GSM8K ($n=6$) & $0.423$ & $0.395$ & $0.440$ & $0.405$ \\
HotpotQA ($n=2$) & $0.436$ & $0.358$ & $0.464$ & $0.337$ \\
ALFWorld ($n=6$) & $0.716$ & $0.005$ & $0.574$ & $0.005$ \\
ScienceWorld ($n=4$) & $0.270$ & $0.183$ & $0.312$ & $0.191$ \\
\bottomrule
\end{tabular}
\end{table}

\Cref{tab:transfer} shows that cross-model transfer barely degrades conservative action-conditioned performance on the matched subset.
On all four benchmarks, the transfer regret of the EU-LCB controller is within $0.05$ of the within-model regret, and on ALFWorld the regret is near-zero in both settings.
The relative ordering between failure-triggered routing and the EU-LCB controller is preserved under transfer on all benchmarks for the subset and estimator of this table; task repetition, features, protocol, and selection are not excluded as contributing factors.

\Needspace{270pt}
\subsection{Scaling curves}

\Cref{fig:scaling} reports control regret as a function of training set size on ALFWorld, ScienceWorld, and GSM8K.
The pattern is consistent with the selection-noise discussion in \Cref{sec:protocols}: the EU-LCB controller overtakes failure-triggered routing earlier on ALFWorld than on ScienceWorld.

\begin{center}
    \includegraphics[width=0.94\linewidth]{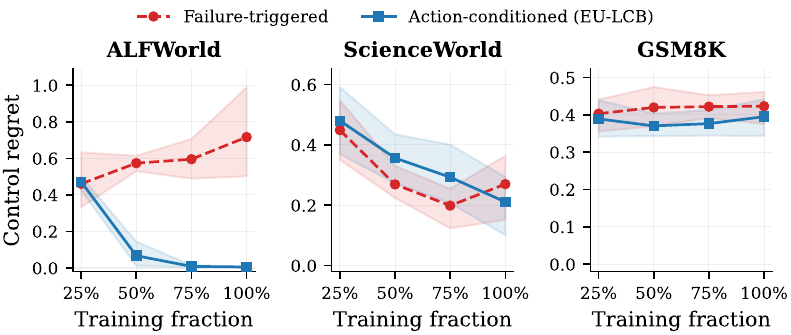}
    \captionsetup{hypcap=false}
    \captionof{figure}{\textbf{Control regret versus training set size.} Lines are run-level means; bands are 95\% bootstrap intervals over runs. The EU-LCB controller overtakes failure-triggered routing earlier on ALFWorld than on ScienceWorld.}
    \label{fig:scaling}
\end{center}

\end{document}